\documentclass[sigconf]{acmart}
\usepackage{multirow}
\usepackage{array}
\usepackage{enumitem}
\usepackage{multirow}
\usepackage{graphicx}
\usepackage{array}
\usepackage{subcaption} 
\usepackage[normalem]{ulem}
\useunder{\uline}{\ul}{}
\AtBeginDocument{%
  }

\useunder{\uline}{\ul}{}

\newtheorem{theorem}{\textbf{Theorem}}[section]
\newtheorem{lemma}{\textbf{Lemma}}[section]
\newtheorem{proposition}{\textbf{Proposition}}[section]

\newtheorem{definition}{\textbf{Definition}}[section]

\copyrightyear{2026}
\acmYear{2026}
\setcopyright{cc}
\setcctype{by}
\acmConference[WWW '26] {Proceedings of the ACM Web Conference 2026}{April 13--17, 2026}{Dubai, United Arab Emirates.}
\acmBooktitle{Proceedings of the ACM Web Conference 2026 (WWW '26), April 13--17, 2026, Dubai, United Arab Emirates}
\acmISBN{979-8-4007-2307-0/2026/04}
\acmDOI{10.1145/3774904.3792077}

\begin{document}
\title{Toward Graph-Tokenizing Large Language Models with Reconstructive Graph Instruction Tuning}

\author{Zhongjian Zhang}
\email{zhangzj@bupt.edu.cn}
\authornote{Both authors contributed equally to this research.}
\affiliation{%
  \institution{Beijing University of Posts and Telecommunications}
  \city{Beijing}
  \country{China}
}
\author{Xiao Wang}
\email{xiao_wang@buaa.edu.cn}
\authornotemark[1]
\affiliation{%
  \institution{Beihang University}
  \city{Beijing}
  \country{China}}
  
\author{Mengmei Zhang}
\email{zhangmengmei@bestpay.com.cn}
\affiliation{
  \institution{China Telecom Bestpay}
  \city{Beijing}
  \country{China}}
  
\author{Jiarui Tan}
\email{tanjiarui@bupt.edu.cn}
\affiliation{
  \institution{Beijing University of Posts and Telecommunications}
  \city{Beijing}
  \country{China}}
  
\author{Chuan Shi}
\email{shichuan@bupt.edu.cn}
\authornote{Corresponding author.}
\affiliation{
  \institution{Beijing University of Posts and Telecommunications}
  \city{Beijing}
  \country{China}}

\begin{abstract}
The remarkable success of large language models (LLMs) has motivated researchers to adapt them as universal predictors for various graph-related tasks, with the ultimate goal of developing a graph foundation model that generalizes diverse scenarios. The key challenge is to align graph data with language spaces so that LLMs can better comprehend graphs. As a popular paradigm, Graph-Tokenizing LLMs (GTokenLLMs) encode complex structures and lengthy texts into a graph token sequence, and then align them with text tokens via language instructions tuning. Despite their initial success, our information-theoretic analysis reveals that existing GTokenLLMs rely solely on text supervision from language instructions, which achieve only implicit graph–text alignment, resulting in a text-dominant bias that underutilizes graph context. To overcome this limitation, we first prove that the alignment objective is upper-bounded by the mutual information between the input graphs and their hidden representations in the LLM, which motivates us to improve this upper bound to achieve better alignment. To this end, we further propose a reconstructive graph instruction tuning pipeline, \textbf{RGLM}\footnote{The source code is available at \href{https://github.com/zhongjian-zhang/RGLM}{https://github.com/zhongjian-zhang/RGLM}}. Our key idea is to reconstruct the graph information from the LLM’s graph token outputs, explicitly incorporating graph supervision to constrain the alignment process. Technically, we embody RGLM by exploring three distinct variants from two complementary perspectives: RGLM-Decoder from the input space; RGLM-Similarizer and RGLM-Denoiser from the latent space. Additionally, we theoretically analyze the alignment effectiveness of each variant. Extensive experiments on various benchmarks and task scenarios validate the effectiveness of the proposed RGLM, paving the way for new directions in GTokenLLMs' alignment research.

\end{abstract}



\maketitle
\begin{figure}[tbp]
    \centering
    \includegraphics[width=0.85\linewidth]{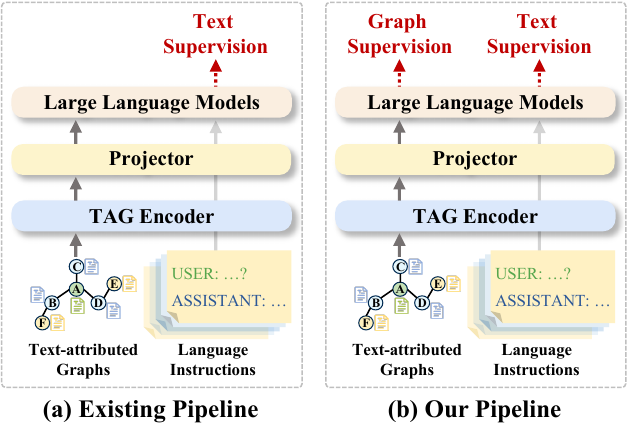}
    \vspace{-15pt}
    \caption{(a) Typical graph instruction tuning methods follow an LLM-centric design that solely leverages text supervision. (b) RGLM designs extra reconstructive graph tasks to explicitly incorporate graph supervision into the tuning process.}
    \label{fig:gtokenllm_pipeline}
    \vspace{-15pt}
\end{figure}
\section{Introduction}
Graph data are ubiquitous in the real world, such as social networks~\cite{huang2024can} and webpage networks~\cite{wu2022semi}.
Traditional graph neural networks (GNNs)~\cite{DBLP:conf/iclr/KipfW17,wu2022nodeformer} are powerful tools for modeling graph structures, but they suffer from a weak multi-task handling capability~\cite{chen2024llaga}.
Recently, the success of GPT-style LLMs~\cite{touvron2023llama,achiam2023gpt,chiang2023vicuna} has motivated numerous studies to adapt them to the graph learning domain, particularly for text-attributed graphs (TAGs), where each node is associated with a text attribute. 
The ultimate goal is to build a graph foundation model that generalizes across diverse scenarios~\cite{liu2025graph}, in which LLMs are typically employed as predictors to accomplish various graph tasks, e.g., node classification, exhibiting promising potential and attracting considerable attention~\cite{chen2024exploring,zhou-etal-2025-graph,zhang2025can,wanggeneralization}.

A central challenge in this endeavor lies in aligning graph data with the natural language embedding space, so that LLMs can better comprehend graphs. To address this challenge, existing efforts mainly fall into two categories: Graph-Textualizing LLMs (GTextLLMs)~\cite{chen2024exploring,fatemitalk,guo2023gpt4graph,wang2023can,zhaographtext} and Graph-Tokenizing LLMs (GTokenLLMs)~\cite{chen2024llaga,tang2024graphgpt,zhang2024graphtranslator,wang2024llms,konggofa}.
The GTextLLMs utilize natural language to describe graphs via manually crafted prompt templates, directly aligning them with text. However, such strategies often suffer from redundancy and imprecise characterization~\cite{chen2024exploring,chen2024llaga}.
In contrast, GTokenLLMs construct various language instructions around graph tasks to perform tuning alignment, demonstrating superior performance while maintaining cost-effective training schemes.
As shown in Figure~\ref{fig:gtokenllm_pipeline}(a), GTokenLLMs typically adopt a TAG encoder to encode the complex structures and lengthy texts into a graph token sequence. This sequence is then aligned with the LLM embedding space via a projector, and subsequently used as prefix tokens for querying LLMs.
However, this pipeline: \textit{TAG encoder} $\rightarrow$ \textit{projector} $\rightarrow$ \textit{LLM} $\Leftarrow$ \textit{language instructions} (where ``$\Leftarrow$'' denotes supervision), is predominantly LLM-centric: (i) graph comprehension largely depends on graph-text alignment and the selected TAG encoder, and (ii) supervision derives exclusively from text data.
This naturally raises a question: 
\textit{Can language instructions alone effectively support graph–text alignment?}
Answering this question is essential for advancing more effective and principled alignment strategies.

Here, we analyze the relationship between language instruction tuning and graph–text alignment from an information-theoretic perspective.
In GTokenLLMs, the ideal optimization objective is to maximize the contribution of graph tokens to generating the current text token, given the preceding text context, which can be realized by maximizing the Conditional Graph–Text Mutual Information (\textbf{C-GTMI}).
We prove that C-GTMI can be decomposed into two main terms (see Theorem~\ref{theory:c-gtmi}): (i) a graph–text alignment term and (ii) a textual autoregressive dependency term. However, we find that the current GTokenLLMs~\cite{chen2024llaga,ye2024language,tang2024graphgpt} primarily optimize the second term, but neglect the explicit optimization of the first term, which achieves only implicit graph-text alignment. This leads to a bias toward textual information, potentially diminishing the significance of graph context in graph comprehension, yielding a suboptimal C-GTMI and limited alignment. 
Therefore, another question naturally arises: \textit{How can we more effectively optimize C-GTMI to achieve better alignment?} 
Considering that LLMs' semantic space is high-dimensional, designing a stable and effective optimization mechanism for this purpose remains a significant challenge.

In this paper, we first prove that the C-GTMI is upper-bounded by the mutual information between the input graph and its hidden representations in the LLM (see Lemma~\ref{lemma:upper_c-gtmi}). 
This result motivates us to improve this upper bound to achieve better alignment, and this improvement can be realized through reconstruction.
To this end, we further propose a novel alignment pipeline, \textbf{RGLM}, termed reconstructive graph instruction tuning. Our key idea is to reconstruct the graph information from the LLM's graph token outputs, which are largely ignored in current GTokenLLMs. As illustrated in Figure 1(b), this pipeline explicitly incorporates graph supervision to constrain the alignment process. Technically, we embody RGLM by exploring three variants from two complementary perspectives: RGLM-Decoder, defined from the input space; RGLM-Similarizer and -Denoiser, defined from the latent space. Additionally, we theoretically analyze how each RGLM variant optimizes the upper bound of the C-GTMI. 
Our main contributions are as follows: \\
$\bullet$ We are the first to theoretically reveal that current GTokenLLMs using language instructions tuning achieve only implicit graph-text alignment, and prove that the alignment objective is upper-bounded by the mutual information between the input graphs and their hidden representations in the LLM.\\
$\bullet$ We propose a reconstructive graph instruction tuning pipeline, called RGLM, and explore three distinct variants: RGLM-Decoder from the input space, RGLM-Similarizer and RGLM-Denoiser from the latent spaces. Additionally, we theoretically analyze the alignment effectiveness of each variant.\\
$\bullet$ Experimental results on four datasets demonstrate that the proposed RGLM, which introduces different reconstructive graph losses, achieves substantial improvements in supervised learning, multi-dataset generalization, and zero-shot capability, highlighting its potential as a graph foundation model.
\section{Preliminary}
\subsection{Text-attributed Graphs (TAGs)} 
We define a TAG as $\mathcal{G}_{\text{full}} \!=\! (\mathcal{V}_{\text{full}}, \mathcal{E}_{\text{full}}, \mathcal{S}_{\text{full}})$, where $\mathcal{V}_{\text{full}} \!=\! \{v_i\}_{i=1}^{|\mathcal{V}_{\text{full}}|}$, $\mathcal{E}_{\text{full}} \!=\! \{e_i\}_{i=1}^{|\mathcal{E}_{\text{full}}|}$, $\mathcal{S}_{\text{full}} \!=\! \{s_i\}_{i=1}^{|\mathcal{S}_{\text{full}}|}$ denote the node, edge, and text set, respectively. The text set $\mathcal{S}_{\text{full}}$ is typically encoded into node features $\mathbf{Z}_{\text{full}}\!=\! \{\mathbf{z}_1, \ldots, \mathbf{z}_{|\mathcal{V}_{\text{full}}|}\}$ using pre-trained language models~\cite{reimers2019sentence,devlin2019bert}. To efficiently process node-level tasks with LLMs, e.g., node classification, it is common to sample a subgraph centered at a node to reduce computational cost and irrelevant information~\cite{chen2024llaga,tang2024graphgpt}. Formally, we denote $\mathcal G_{\text{sub}}^{h}(v) \!\!=\!\! (\mathcal{V}_v^h, \mathcal{E}_v^h, \mathbf{Z}_v^h)$ as the $h$-hop subgraph centered at node $v$, where $\mathcal{V}_v^h$ includes all nodes within $h$ hops of $v$, $\mathcal{E}_v^h$ contains their interconnecting edges, and $\mathbf{Z}_v^h$ is the corresponding feature matrix. 
For simplicity, we use $\mathcal{G}\!=\!\{\mathcal{V}, \mathcal{E}, \mathbf{Z}\}$ to denote such subgraphs when clear from context.


\begin{figure}[tp]
    \centering
    \includegraphics[width=0.82\linewidth]{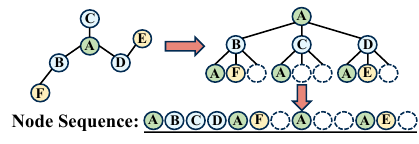}
    \vskip -0.2in
    \caption{Neighbor detail template for graph serialization~\cite{chen2024llaga}.}
    \vskip -0.225in
    \label{fig:ndt}
\end{figure}
\subsection{Graph-Tokenizing LLMs (GTokenLLMs)}
Generally, a $\theta$-parameterized LLM models the canonical causal distribution over a text token sequence $\boldsymbol{x}_{1:T}=(\boldsymbol{x}_1, \dots, \boldsymbol{x}_T)$ of length $T$ as $p_\theta(\boldsymbol{x})=\prod_{i=1}^Tp_\theta(\boldsymbol{x}_i|\boldsymbol{x}_{<i})$. 
To enable LLMs to understand graph content, GTokenLLMs encode complex structures and lengthy texts into a graph token sequence, and treat them as prefix tokens for querying LLMs. 
As illustrated in Figure~\ref{fig:overview}(a), a graph $\mathcal{G}$ is first encoded as a token sequence via a TAG encoder $\mathcal{F}$, where $\mathcal{F}$ can be a predefined GNN~\cite{zhang2024graphtranslator,konggofa,he2025unigraph} or a graph serialization template~\cite{chen2024llaga,tang2024graphgpt,ye2024language}. These tokens are subsequently projected into the LLM’s input space via a $\phi$-parameterized projector $\mathcal{P}$. Thus, the canonical causal distribution is reformulated as:
\begin{align}
p_\Theta(\boldsymbol{x})=\prod_{i=1}^Tp_\Theta(\boldsymbol{x}_i|\boldsymbol{v}, \boldsymbol{x}_{<i}),\quad\boldsymbol{v}=\mathcal{P}_\phi\circ\mathcal{F}\left(\mathcal{G}\right),
\end{align}
where $\Theta=\{\theta, \phi\}$ is the parameters, $\boldsymbol{v}_{1:N}=(\boldsymbol{v}_1,\dots,\boldsymbol{v}_N)$ is the graph token sequence of length $N$.
In this paper, we adopt the state-of-the-art graph serialization Neighbor Detail Template~\cite{chen2024llaga} as $\mathcal{F}$, owing to its elaborate descriptions of a node and its surrounding neighborhood, which is more suitable for our subsequent reconstruction task.
As illustrated in Figure~\ref{fig:ndt}, given a central node $A$, we construct a fixed-size computational tree rooted at $A$ to represent its $h$-hop subgraph. In this tree, $A$ serves as the root, its 1-hop neighbors are assigned as its children, 2-hop neighbors as the children of 1-hop neighbors, and so on. When a node has more neighbors than the predefined size, we randomly sample a subset; if it has fewer neighbors, we insert placeholder nodes to maintain structural consistency. The tree is then traversed in level-order to produce a fixed-length graph token sequence, where each token corresponds to a node and is represented by its feature $\mathbf{z}$. This process effectively transforms the graph structure into a sequence while preserving the hierarchical neighborhood information.

Typically, GTokenLLMs are trained with supervised instruction tuning, where the model parameters $\Theta$ are optimized by minimizing an autoregressive loss on language supervision (e.g., textual labels of node classification). 
In this setting, the autoregressive loss enforces next-token prediction, i.e., each input token yields one output conditioned on the preceding context. 
The overall objective is to maximize the log-likelihood of the supervised textual outputs.
\begin{align}
\label{eq:text_loss}
\mathcal{L}_{\mathrm{text}}\left(\Theta=\{\theta,\phi\},\mathcal{G},\boldsymbol{x}\right) 
= \frac{-1}{|\mathcal{S}|}\sum_{i\in\mathcal{S}}\log p_\Theta(\boldsymbol{x}_i \!\mid\! \boldsymbol{v}, \boldsymbol{x}_{<i}),
\end{align}
where $\mathcal{S}\subseteq \{N+1,\dots, T\}$ denotes the index set of textual label tokens in the input sequence. As indicated in Equation~(\ref{eq:text_loss}), only the text outputs $\boldsymbol{x}_{i \in \mathcal{S}}$ are supervised.

\section{Motivation Analysis}\label{sec:motivation}
\subsection{Implicit Graph-Text Alignment of Existing GTokenLLMs}
Given an input token sequence of length $T$, consisting of a prefix graph token sequence $\boldsymbol{v}_{1:N}=(\boldsymbol{x}_1, \dots, \boldsymbol{x}_N)$ and a text token sequence $\boldsymbol{x}_{1:M}=(\boldsymbol{x}_1, \dots, \boldsymbol{x}_M)$, where $M=T-N$.
At decoding step $t$, GTokenLLMs predict $\boldsymbol{x}_t$ conditioned on $\boldsymbol{v}$ and $\boldsymbol{x}_{<t}$. We denote entropy by $\mathcal{H}(\cdot)$ and mutual information by $\mathcal{I}(\cdot ; \cdot)$. The ideal optimization objective of GTokenLLMs is to maximize the contribution of graph tokens to generating the current text token, given the preceding text context. To formalize it, we first give the definition:
\begin{definition}[Conditional Graph-Text Mutual Information (C-GTMI)]
At decoding step $t$, let $\boldsymbol{s}^{G}$ and $\boldsymbol{s}^{T}_{<t}$ denote the LLM’s hidden representation of the graph tokens and the preceding text tokens, respectively.
The C-GTMI is defined as $\mathcal{I}(\boldsymbol{x}_t;\boldsymbol{s}^{G}\!\mid\! \boldsymbol{s}^{T}_{<t})$, 
which measures how much the graph tokens contribute to generating the current text token given the preceding text context.
\end{definition}
Based on the above definition, the optimization objective of GTokenLLMs can be realized by maximizing the C-GTMI.
\begin{theorem}[Implicit Graph-Text Alignment via Autoregressive Loss]\label{theory:c-gtmi}
The C-GTMI can be decomposed as:
\begin{equation}
\label{eq:chain}
\mathcal{I}(\boldsymbol{x}_t;\boldsymbol{s}^{G}\!\mid\!\boldsymbol{s}_{<t}^{T})
=\mathcal{I}(\boldsymbol{x}_t;\boldsymbol{s}^{G})
\;+\; \mathcal{I}(\boldsymbol{x}_t;\boldsymbol{s}_{<t}^{T}\!\mid\! \boldsymbol{s}^{G})
\;-\; \mathcal{I}(\boldsymbol{x}_t;\boldsymbol{s}_{<t}^{T}) ,
\end{equation}
where $\mathcal{I}(\boldsymbol{x}_t;\boldsymbol{s}^{G})$ reflects graph-text alignment, $\mathcal{I}(\boldsymbol{x}_t;\boldsymbol{s}_{<t}^{T}\!\mid\!\boldsymbol{s}^{G})$ captures textual autoregressive dependencies, and $\mathcal{I}(\boldsymbol{x}_t;\boldsymbol{s}_{<t}^{T})$ depends only on the textual data distribution. 
\end{theorem}
Theorem~\ref{theory:c-gtmi} reveals that the autoregressive loss on language supervision primarily optimizes the textual autoregressive dependency, but only implicitly encourages graph-text alignment, namely that $\mathcal{I}(\boldsymbol{x}_t;\boldsymbol{s}^{G})$ is not explicitly optimized. Consequently, the model tends to overly rely on textual information, potentially diminishing the significance of graph context in graph comprehension, yielding a suboptimal C-GTMI and limited performance.
\begin{proof}
In practice, the autoregressive training objective of GTokenLLMs is the cross-entropy loss 
between the true distribution $p$ and the model distribution $q_\theta$. 
Equivalently, at step $t$, this can be written as the negative log-likelihood:
\begin{equation}
\begin{aligned}
\label{eq:ce}
\mathcal{L}_{\mathrm{CE}}
&=-\mathbb{E}\left[\log q_{\theta}\left(\boldsymbol{x}_t\!\mid\! \boldsymbol{v}, \boldsymbol{x}_{<t}\right)\right] \\
&=\mathcal{H}(\boldsymbol{x}_t\!\mid\! \boldsymbol{v},\boldsymbol{x}_{<t})
+\mathbb{E}\left[\mathrm{KL}p(\boldsymbol{x}_t\!\mid\!\cdot)\parallel q_\theta(\boldsymbol{x}_t\!\mid\!\cdot)\right] \\
&\ge\mathcal{H}\left(\boldsymbol{x}_t\!\mid\! \boldsymbol{v},\boldsymbol{x}_{<t}\right).
\end{aligned}
\end{equation}
Since decoder-only LLMs generate tokens based on their internal hidden states rather than directly on raw input tokens, the conditional probability of the next text token can be equivalently expressed in terms of these hidden representations. Thus, we have:
\begin{equation}
\label{eq:sufficient}
q_{\theta}(\boldsymbol{x}_t\!\mid\! \boldsymbol{v},\boldsymbol{x}_{<t})
= q_\theta(\boldsymbol{x}_t\!\mid\! \boldsymbol{s}^{G},\boldsymbol{s}_{<t}^T).
\end{equation}
This formulation highlights that the prediction of the next token depends solely on the internal hidden representations of both modalities, providing a foundation for analyzing alignment at the level of hidden states. Substituting this into Equation~(\ref{eq:ce}) yields:
\begin{equation}
\begin{aligned}
\label{eq:ce-avg}
\mathcal{L}_{\mathrm{CE}}
= -\mathbb{E}[\log q_{\theta}(\boldsymbol{x}_t\!\mid\! \boldsymbol{s}^{G}, \boldsymbol{s}^{T}_{<t})] \ge \mathcal{H}(\boldsymbol{x}_t \!\mid\! \boldsymbol{s}^{G}, \boldsymbol{s}^{T}_{<t}).
\end{aligned}
\end{equation}
According to the mutual information identity:
\begin{align}
\mathcal{H}(\boldsymbol{x}_t\!\mid\! \boldsymbol{s}^{G},\boldsymbol{s}_{<t}^{T})
=\mathcal{H}(\boldsymbol{x}_t\!\mid\! \boldsymbol{s}_{<t}^{T})-\mathcal{I}(\boldsymbol{x}_t;\boldsymbol{s}^{G}\!\mid\! \boldsymbol{s}_{<t}^{T}).
\end{align}
Intuitively, $\mathcal{H}(\boldsymbol{x}_t\!\!\mid\!\!\boldsymbol{s}_{<t}^{T})$ is determined by the data distribution, and minimizing $\mathcal{H}(\boldsymbol{x}_t\!\!\mid\!\!\boldsymbol{s}^G, \boldsymbol{s}_{<t}^T)$ is equivalent to maximizing the C-GTMI $\mathcal{I}(\boldsymbol{x}_t; \boldsymbol{s}^G\!\mid\!\boldsymbol{s}_{<t}^T)$. By the chain rule of mutual information, we have:
\begin{equation}
\mathcal{I}(\boldsymbol{x}_t;\boldsymbol{s}^{G}\!\mid\!\boldsymbol{s}_{<t}^{T})
=\mathcal{I}(\boldsymbol{x}_t;\boldsymbol{s}^{G})
\;+\; \mathcal{I}(\boldsymbol{x}_t;\boldsymbol{s}_{<t}^{T}\!\mid\! \boldsymbol{s}^{G})
\;-\; \mathcal{I}(\boldsymbol{x}_t;\boldsymbol{s}_{<t}^{T}).
\end{equation}


\end{proof}

\vspace{-13pt}
\subsection{Towards Explicit Graph–Text Alignment}
To achieve better alignment by more effectively optimizing C-GTMI, we first theoretically establish an upper bound for C-GTMI:
\begin{lemma}[Upper Bound of C-GTMI]\label{lemma:upper_c-gtmi}
Given an input graph $\mathcal{G}$, the C-GTMI is upper-bounded by the mutual information between the graph $\mathcal{G}$ and its hidden representations $\boldsymbol{s}^G$ in the LLM, i.e.,
\begin{align}
\mathcal{I}(\boldsymbol{x}_t; \boldsymbol{s}^G \!\mid\! \boldsymbol{s}^{T}_{<t})
\le \mathcal{I}(\mathcal{G}; \boldsymbol{s}^G).
\end{align}
\end{lemma}
Lemma~\ref{lemma:upper_c-gtmi} motivates us to optimize the lower bound of $\mathcal{I}(\mathcal{G}; \boldsymbol{s}^G)$, and this improvement can be further realized through reconstruction~\cite{chen2016infogan,hjelmlearning,wangreconstructive,wang2025ross3d}, which in turn improves the attainable upper bound of C-GTMI. This not only explicitly optimizes the graph–text alignment term $\mathcal{I}(\boldsymbol{x}_t;\boldsymbol{s}^{G})$, but also enhances autoregressive dependency modeling $\mathcal{I}(\boldsymbol{x}_t;\boldsymbol{s}{<t}^{T}\!\mid\!\boldsymbol{s}^{G})$, thereby achieving better alignment.
\begin{proof}
Based on the chain rule of mutual information, we first analyze the $\mathcal{I}(\mathcal{G}, \boldsymbol{s}^G, \boldsymbol{x}_t)$ under the condition $\boldsymbol{s}^{T}_{<t}$ and have:
\begin{equation}
\begin{aligned}
\mathcal{I}(\mathcal{G};\boldsymbol{s}^{G},\boldsymbol{x}_t\!\mid\! \boldsymbol{s}^{T}_{<t})
&=\mathcal{I}(\mathcal{G};\boldsymbol{s}^{G}\!\mid\! \boldsymbol{s}^{T}_{<t})
+\mathcal{I}(\mathcal{G};\boldsymbol{x}_t\!\mid\! \boldsymbol{s}^{G},\boldsymbol{s}^{T}_{<t}) \\
&=\mathcal{I}(\mathcal{G};\boldsymbol{x}_t\!\mid\! \boldsymbol{s}^{T}_{<t})
+\mathcal{I}(\mathcal{G};\boldsymbol{s}^{G}\!\mid\! \boldsymbol{x}_t,\boldsymbol{s}^{T}_{<t}).
\end{aligned}
\end{equation}
By equating the two expressions and rearranging terms, and then applying the symmetry of mutual information, we obtain:
\begin{equation}
\begin{aligned}
\mathcal{I}(\mathcal{G};\boldsymbol{s}^{G}\!\mid\! \boldsymbol{s}^{T}_{<t})
&- \mathcal{I}(\mathcal{G};\boldsymbol{s}^{G}\!\mid\! \boldsymbol{x}_t,\boldsymbol{s}^{T}_{<t}) \\
&= \mathcal{I}(\mathcal{G};\boldsymbol{x}_t\!\mid\! \boldsymbol{s}^{T}_{<t})
- \mathcal{I}(\mathcal{G};\boldsymbol{x}_t\!\mid\! \boldsymbol{s}^{G},\boldsymbol{s}^{T}_{<t}) \\
&= \mathcal{I}(\boldsymbol{x}_t;\mathcal{G}\!\mid\! \boldsymbol{s}^{T}_{<t})
- \mathcal{I}(\boldsymbol{x}_t;\mathcal{G}\!\mid\! \boldsymbol{s}^{G},\boldsymbol{s}^{T}_{<t}) \\
&=\mathcal{I}(\boldsymbol{x}_t;\boldsymbol{s}^{G}\!\mid\! \boldsymbol{s}^{T}_{<t})-\mathcal{I}(\boldsymbol{x}_t;\boldsymbol{s}^{G}\!\mid\! \mathcal{G},\boldsymbol{s}^{T}_{<t}).
\end{aligned}
\end{equation}
Given that graph tokens $\boldsymbol{v}$ are typically input as prefix tokens into decoder-only LLMs, the $\boldsymbol{s}^G$ are deterministically derived from $\mathcal{G}$, thus we have $\mathcal{I}(\boldsymbol{x}_t;\boldsymbol{s}^{G}\!\mid\! \mathcal{G},\boldsymbol{s}^{T}_{<t})=0$,
yielding the identity:
\begin{equation}
\begin{aligned}
\mathcal{I}(\boldsymbol{x}_t;\boldsymbol{s}^{G}\!\mid\! \boldsymbol{s}^{T}_{<t})
&=\mathcal{I}(\mathcal{G};\boldsymbol{s}^{G}\!\mid\! \boldsymbol{s}^{T}_{<t})
-\mathcal{I}(\mathcal{G};\boldsymbol{s}^{G}\!\mid\! \boldsymbol{x}_t,\boldsymbol{s}^{T}_{<t}) \\
&=\mathcal{H}(\boldsymbol{s}^{G}\!\mid\! \boldsymbol{s}^{T}_{<t})-\mathcal{I}(\mathcal{G};\boldsymbol{s}^{G}\!\mid\! \boldsymbol{x}_t,\boldsymbol{s}^{T}_{<t})\\
&\le \mathcal{H}(\boldsymbol{s}^G \!\mid\! \boldsymbol{s}^{T}_{<t}).
\end{aligned}
\end{equation}
Since $\mathcal{I}(\boldsymbol{s}^G; \boldsymbol{s}^{T}_{<t}) 
= \mathcal{H}(\boldsymbol{s}^G) - \mathcal{H}(\boldsymbol{s}^G \!\mid\! \boldsymbol{s}^{T}_{<t}) \ge 0$, it follows:
\begin{align}
\mathcal{H}(\boldsymbol{s}^G \!\mid\! \boldsymbol{s}^{T}_{<t}) \le \mathcal{H}(\boldsymbol{s}^G).
\end{align}
Moreover, since $\boldsymbol{s}^G$ is a deterministic function of $\mathcal{G}$:
\begin{align}
\mathcal{I}(\mathcal{G}; \boldsymbol{s}^G) = \mathcal{H}(\boldsymbol{s}^G) - \mathcal{H}(\boldsymbol{s}^G \!\mid\! \mathcal{G}) = \mathcal{H}(\boldsymbol{s}^G).
\end{align}
Combining the inequalities yields the final result:
\begin{align}
\mathcal{I}(\boldsymbol{x}_t; \boldsymbol{s}^G \!\mid\! \boldsymbol{s}^{T}_{<t})
\le \mathcal{I}(\mathcal{G}; \boldsymbol{s}^G).
\end{align}
\end{proof}

\vspace{-15pt}
\section{Methodology}

\subsection{Overview}
Motivated by Lemma~\ref{lemma:upper_c-gtmi}, we propose a reconstructive graph instruction tuning pipeline, called RGLM, which explicitly optimizes the lower bound of $\mathcal{I}(\mathcal{G}; \boldsymbol{s}^G)$ through extra reconstructive graph tasks.
As shown in Figure~\ref{fig:overview}(a), a graph $\mathcal{G}$ is first serialized into a token sequence using the neighbor detail template $\mathcal{F}$. These tokens are aligned with the text space via a $\phi$-parameterized projector, and then used as prefix tokens to query the $\theta$-parameterized LLM.
Our RGLM aims to reconstruct graph information from the LLM-generated graph token representation $\boldsymbol{s}^{G}$. Since $\boldsymbol{s}^{G}$ may contain multiple tokens for the same node (e.g., node $A$ appears four times in Figure~\ref{fig:ndt}), we apply mean aggregation over these tokens to obtain the reconstruction token representation $\mathbf{H}=\{\mathbf{h}_1, \cdots, \mathbf{h}_{|\mathcal{V}|}\}$:
\begin{equation}
\mathbf{h}_v = \frac{1}{|\Gamma(v)|} \sum_{i \in \Gamma(v)} \boldsymbol{s}^{G}_i, 
\quad \forall v \in \mathcal{V},
\end{equation}
where $\Gamma(v) \subseteq \{1,\ldots,N\}$ denotes the index set of tokens associated with node $v$ in $\boldsymbol{s}^{G}$. 
Here, to reduce computational cost and avoid catastrophic forgetting~\cite{kirkpatrick2017overcoming}, we adopt LoRA~\cite{hu2022lora} for fine-tuning the LLM. The training objective comprises two components: 
\begin{align}
\mathcal{L}_{\text{RGLM}} = \mathcal{L}_{\text{text}} + \mathcal{L}_{\text{graph}},
\end{align}
where $\mathcal{L}_{\text{text}}$ is the textual autoregressive loss defined in Equation~(\ref{eq:text_loss}), and $\mathcal{L}_{\text{graph}}$ is an extra reconstruction term on graph data. The $\mathcal{L}_{\text{graph}}$ is defined using a customized metric $\mathcal{M}$ that measures the discrepancy between the reconstruction tokens $\mathbf{H}$ and a predefined target derived from the graph $\mathcal{G}$. Formally, this loss is expressed as:
\begin{align}\label{eq:reconstrucion}
    \mathcal{L}_{\text{graph}}(\Theta=\{\phi, \theta, \pi\}, \mathcal{G}, \boldsymbol{x})=\mathcal{M}\left( \mathcal{J}_{\pi}\left(\mathbf{H}\right), \mathcal{R}\left(\mathcal{G}\right)\right),
\end{align}
where $\mathcal{J}_{\pi}$ is a $\pi$-parameterized post-projection function, which maps $\mathbf{H}$ to be consistent with the predefined target. The function $\mathcal{R}$ constructs the reconstruction target from $\mathcal{G}$, which may include any observable graph signals. 
Evidently, the choice of $\mathcal{R}$ and $\mathcal{M}$ yields different variants.
In the following, we explore three variants with alignment guarantees from two complementary perspectives. In the input space, we consider a variant that directly reconstructs the raw node features and topology.
In the latent space of pre-trained GNNs, we examine two semantic variants that reconstruct latent node representations: one uses a similarity-based objective while the other employs a denoising strategy.
\begin{figure*}[tp]
    \centering
    \includegraphics[width=1\linewidth]{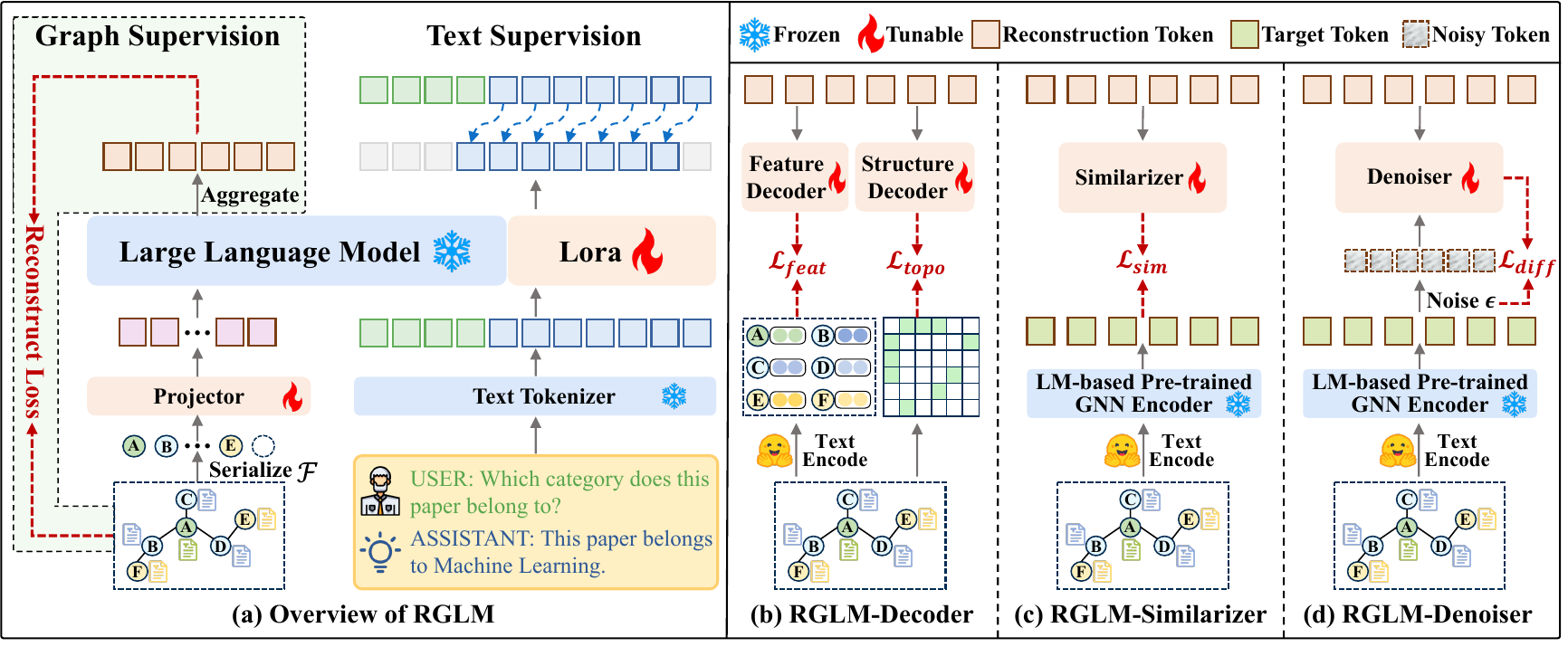}
    \vskip -0.15in
    \caption{(a) Overview of RGLM. Given an input TAG, RGLM aims to reconstruct the graph information from the LLM's graph token outputs, explicitly incorporating graph supervision to constrain the alignment process. (b) Decoder variant directly reconstructs the raw node features and topology in the input space. (c) Similarizer and (d) Denoiser variants reconstruct latent node representations of pre-trained GNNs via the cosine-similarity and denoising strategy, respectively.
    }
    \vskip -0.15in
    \label{fig:overview}
\end{figure*}
\subsection{Raw Graph Reconstruction in Input Space}
In the input space, we introduce RGLM-Decoder, which reconstructs the raw graph information from $\mathbf{H}$.
As shown in Figure~\ref{fig:overview}(b), we train a feature decoder $d_f$ and a structure decoder $d_s$ to reconstruct the node features and the adjacency, respectively. In this setting, $\mathcal{J}_{\pi}$ is instantiated by $d_f$ and $d_s$, while $\mathcal{R}$ consists of a pre-trained language model for text encoding and an adjacency matrix extractor. The metric $\mathcal{M}$ is defined as the mean squared error (MSE) for feature reconstruction and the binary cross entropy (BCE) for structure reconstruction. The objective function is formulated as:
\begin{align}
\mathcal{L}_{\text{graph}} 
&= \lambda_f \cdot \mathcal{L}_{\text{feat}} 
+ \lambda_s \cdot \mathcal{L}_{\text{topo}}, \\
\mathcal{L}_{\text{feat}} 
&= \frac{1}{|\mathcal{V}|} \sum_{i=1}^{|\mathcal{V}|} 
   \big\| d_f(\mathbf{h}_i) - \mathbf{z}_i \big\|_2^2, \\
\mathcal{L}_{\text{topo}} 
&= - \frac{1}{|\mathcal{E}|} \sum_{(i,j)\in\mathcal{E}} 
      \log \sigma \big(d_s(\mathbf{h}_i,\mathbf{h}_j)\big) \nonumber \\
&\quad - \frac{1}{|\hat{\mathcal{E}}|} \sum_{(i,j')\in\hat{\mathcal{E}}} 
      \log \big( 1 - \sigma(d_s(\mathbf{h}_i,\mathbf{h}_{j'})) \big),
\end{align}
where $\mathcal{E}$ and $\hat{\mathcal{E}}$ represent sets of existing and non-existing connections in $\mathcal{G}$, respectively, and $\mathbf{z}$ denotes the node feature. $\sigma(\cdot)$ is the sigmoid function, $\lambda_f$ and $\lambda_s$ balance the feature and structural reconstruction losses, respectively. 

\begin{proposition}[Lower Bound via RGLM-Decoder] \label{prop:raw_graph_recon}
Given a graph $\mathcal{G}=(\mathbf{Z}, \mathbf{A})$, where $\mathbf{Z}$ denotes the node feature matrix and $\mathbf{A}$ the adjacency matrix, minimizing the reconstruction losses $\mathcal{L}_{\text{feat}}$ and $\mathcal{L}_{\text{topo}}$ provides a variational lower bound of $\mathcal{I}(\mathcal{G}; \boldsymbol{s}^G)$:
\begin{align}
\mathcal{I}(\mathcal{G};\boldsymbol{s}^{G})
\ \ge\ \mathcal{H}(\mathbf{Z})+\mathcal{H}(\mathbf{A})
\ -\ \tfrac{1}{\lambda_f}\mathcal{L}_{\text{feat}}
\ -\ \tfrac{1}{\lambda_s}\mathcal{L}_{\text{topo}}.
\end{align}
\end{proposition}
The proof is given in Appendix~\ref{proof:raw_graph_recon}. Proposition~\ref{prop:raw_graph_recon} indicates minimizing the raw graph reconstruction loss maximizes a computable lower bound of $\mathcal{I}(\mathcal{G};\boldsymbol{s}^{G})$, which directly improves the upper bound of optimization objective C-GTMI $\mathcal{I}(\boldsymbol{x}_t;\boldsymbol{s}^{G}\!\mid\! \boldsymbol{s}^{T}_{<t})$.

\subsection{Node Representation Reconstruction in Latent Space}
Unlike reconstructing the raw graph information in the input space, this subsection focuses on reconstructing node representation in the latent space of pre-trained GNNs. Specifically, we first pre-train an LM-based GNN encoder $g$ to obtain fine-grained node representations, and then train a lightweight module, Similarizer or Denoiser, to reconstruct node-level semantics from the $\mathbf{H}$.  


\paragraph{\textbf{LM-based Pre-trained GNN Encoder.}} 
To obtain fine-grained node representations that effectively capture discrete graph structures and multi-level feature information (i.e., node-, edge-, and graph-level), we pretrain an LM-based GNN encoder $g$ within an autoencoder framework~\cite{yangvqgraph,zhou2024unifying}.
Specifically, given a full TAG $\mathcal{G}_{\text{full}} = (\mathcal{V}_{\text{full}}, \mathcal{E}_{\text{full}}, \mathbf{S}_{\text{full}})$, we first encode the textual attributes $\mathbf{S}_{\text{full}}$ using Sentence-BERT~\cite{reimers2019sentence} to extract the initial node features $\mathbf{Z}_{\text{full}}$. Following~\cite{zhou2024unifying,rampavsek2022recipe}, to further enhance structural awareness, we incorporate learnable edge embeddings by treating absent edges as a special edge type, while introducing LapPE and RWSE positional encodings. We adopt the combination of DeeperGCN~\cite{li2020deepergcn} and GAT~\cite{velivckovic2018graph} as the backbone of the encoder. Since the latent space is already powerful, the decoder can be a relatively simple linear layer. In the training stage, we randomly mask node labels of the training set for prediction, which can be regarded as a special case of graph reconstruction. The objective function is defined as follows:
\begin{equation}
\mathcal{L}_{\text{pretrain}} = \mathcal{L}_{\text{mask}} + \mathcal{L}_{\text{reg}},
\end{equation}
where $\mathcal{L}_{\text{mask}}$ denotes the loss for masked label prediction, and $\mathcal{L}_{\text{reg}}$ is a regularization term, implemented as a KL-penalty to push the learned latent representations towards a standard normal distribution~\cite{yangvqgraph}. Once the encoder $g$ is well trained, it is used to encode each subgraph $\mathcal{G}$ and produce the target token representations $\mathbf{E}$, where $g$ corresponds to the reconstruction function $\mathcal{R}$ in Equation~(\ref{eq:reconstrucion}).

\paragraph{\textbf{RGLM-Similarizer.}} As shown in Figure~\ref{fig:overview}(c), RGLM-Similarizer trains a lightweight multilayer perceptron $s$, referred to as the Similarizer, to directly regress the reconstruction token representations $\mathbf{H}$ onto the fine-grained target token representations $\mathbf{E}$. In this setting, $\mathcal{M}$ is defined as the cosine similarity loss, and $\mathcal{J}$ corresponds to the $s$. The objective function is formulated as:

\begin{equation}\label{eq:l_sim}
\mathcal{L}_{\text{graph}}= \lambda_{l}\cdot\mathcal{L}_{\text{sim}}, \quad
\mathcal{L}_{\text{sim}} = 1 - \frac{\langle s(\mathbf{H}), \mathbf{E} \rangle}{\| s(\mathbf{H}) \|_2 \, \| \mathbf{E} \|_2},
\end{equation}
where $\lambda_{l}$ is a scaling factor that balances the reconstruction losses.
\begin{proposition}[Lower Bound via RGLM-Similarizer] \label{prop:latent_sim}
Given the latent node representations $\mathbf{E}$ of a graph $\mathcal{G}$ obtained from a pre-trained GNN $g$,
minimizing the cosine similarity loss $\mathcal{L}_{\text{sim}}$ provides a variational lower bound of $\mathcal{I}(\mathcal{G}; \boldsymbol{s}^G)$:
\begin{equation}
\mathcal{I}(\mathcal{G};\boldsymbol{s}^{G})\ge \mathcal{H}(\mathbf{E})-(\kappa\cdot\lambda_{l}\cdot\mathcal{L}_{\text{sim}}+\mathrm{C}).
\label{eq:sim-mi-lb}
\end{equation}
where $\kappa>0$ and $\mathrm{C}$ is a constant independent of $\boldsymbol{s}^{G}$. The proof is given in Appendix~\ref{proof:latent_sim}.
\end{proposition}
\begin{table*}[tp]
\vskip -0.05in
\caption{Node classification results under supervised settings. Best are bolded, and runner-up are underlined. ``OOM'' indicates out-of-memory, exceeding the A800's 80 GB limit. ``Rank'' represents the overall ranking on all datasets.}\label{tab:main_node}
\vskip -0.15in
\setlength{\extrarowheight}{-0.35pt}
\resizebox{\linewidth}{!}{
\small
\begin{tabular}{c|c|cc|cc|cc|cc|cc}
\hline
\multirow{2}{*}{Method}     & Dataset          & \multicolumn{2}{c|}{Cora}       & \multicolumn{2}{c|}{Pubmed}     & \multicolumn{2}{c|}{OGBN-Arxiv} & \multicolumn{2}{c|}{Reddit}  & \multicolumn{2}{c}{Rank}    \\ \cline{2-12} 
                            & Metric           & Acc            & F1             & Acc            & F1             & Acc            & F1             & Acc            & F1             & Acc            & F1 \\ \hline
\multirow{9}{*}{GNNs}       & GCN              & 89.30          & 87.95          & 89.48          & 89.04          & 73.83          & 54.97          & 63.19          & 62.49     & 5          & 6     \\
                            & GraphSAGE        & 87.82          & 86.77          & 90.54          & 90.51          & 74.47          & 56.41          & 58.51          & 58.41          & 10          & 8 \\
                            & GAT              & 87.08          & 86.32          & 88.11          & 87.60          & 74.39          & 57.71          & 61.78          & 65.38          & 9          & 5 \\ \cline{2-12} 
                            & DIFFormer        & 85.60          & 84.31          & 81.60          & 81.94          & 58.80          & 42.44          & 55.30          & 55.09          & 15          & 14 \\
                            & UniMP            & 87.82          & 86.32          & 90.34          & 90.17          & 70.96          & 53.24          & 59.55          & 59.53          & 11          & 10 \\
                            & NodeFormer       & 87.30          & 85.03          & 89.92          & 89.63          & 68.69          & 50.64          & {\ul 68.06}    & {\ul 67.89}    & 7          & 7 \\ \cline{2-12} 
                            & GraphCL          & 86.53          & 84.50          & 90.39          & 90.28          & OOM            & OOM            & 57.67          & 57.67          & 12          & 12 \\
                            & GRACE            & 89.11          & 88.12          & 86.89          & 86.58          & 72.00          & 50.79          & 65.39          & 65.39          & 8          & 9 \\
                            & LGD              & 85.61          & 84.61          & 90.77          & 90.65          & 72.88          & 50.06          & 65.92          & 62.00          & 6          & 11 \\ \hline
\multirow{1}{*}{GTextLLMs}  & GraphText-ICL             & 82.29          & 81.16          & 56.72          & 52.55          & 68.67          & 51.67          & 56.50          & 52.23          & 16          & 16 \\ \hline
\multirow{5}{*}{GTokenLLMs} & GraphText-SFT         & 84.69          & 83.77          & 87.80          & 87.68          & 49.88                             & 19.46          & 62.20          & 62.07          & 14          & 15  \\
                            & GraphGPT         & 83.92          & 82.72          & 81.76          & 80.99          & 71.59          & 54.58          & 62.11          & 62.39          & 13          & 13 \\
                            & LLaGA            & 88.75          & 87.87          & 90.34          & 90.25          & 74.61          & 56.48          & 66.38          & 66.38          & 4          & 4 \\ \cline{2-12} 
                            & RGLM-Decoder     & 89.85          & \textbf{89.39} & {\ul 91.15}    & {\ul 91.07}    & 75.00          & {\ul 56.92}    & \textbf{68.64} & \textbf{68.54}  & 1          & 1 \\
                            & RGLM-Similarizer & {\ul 89.67}    & 88.96          & \textbf{91.61} & \textbf{91.51} & \textbf{75.14} & \textbf{57.74} & 66.79          & 66.77          & 3          & 2 \\
                            & RGLM-Denoiser    & \textbf{90.22} & {\ul 89.35}    & 90.95          & 90.90          & {\ul 75.10}    & 56.84          & 67.64          & 67.50          & 2          & 3 \\ \hline
\end{tabular}
}
\vskip -0.145in
\end{table*}
\paragraph{\textbf{RGLM-Denoiser.}} 
RGLM-Similarizer is a result-oriented reconstruction, which directly regresses $\mathbf{H}$ onto $\mathbf{E}$. In contrast, RGLM-Denoiser is a process-oriented reconstruction by leveraging a denoising objective. As illustrated in Figure~\ref{fig:overview}(d), RGLM-Denoiser trains a denoiser $f$ that conditions on the reconstruction tokens $\mathbf{H}$ to recover clean latent tokens $\mathbf{E}$ from their noisy tokens $\mathbf{E}_t$. The denoiser $f$ is implemented as a stack of transformer blocks, each equipped with three additional projection layers for the conditions $\mathbf{H}$, the noisy inputs $\mathbf{E}_t$, and the timesteps $t$, respectively. In this setting, $\mathcal{M}$ is defined as the MSE loss between the predicted and true noise, and $\mathcal{J}$ corresponds to the $f$.
The training procedure follows a standard diffusion process:
\begin{equation}
\mathcal{L}_{\text{graph}}= \lambda_{l}\cdot\mathcal{L}_{\text{diff}}, \quad
\mathcal{L}_{\text{diff}} = \mathbb{E}_{t, \boldsymbol{\epsilon}} 
\left[ \left\| f(\mathbf{E}_t, \mathbf{H}, t) - \boldsymbol{\epsilon} \right\|_2^2 \right],
\end{equation}
where $\lambda_{l}$ is a scaling factor that balances the reconstruction losses. $\mathbf{E}_t$ is sampled from $\mathcal{N}(\sqrt{1-\beta_t}\mathbf{E}_{t-1}, \beta_t \mathbf{I})$, 
and $\mathbf{E}_{t-1}$ can be reparameterized directly from $\mathbf{E}$ by letting 
$\bar{\alpha}_t=\prod_{i=1}^t(1-\beta_i)$:
\begin{equation}
\mathbf{E}_t = \sqrt{\bar{\alpha}_t}\mathbf{E} + \sqrt{1-\bar{\alpha}_t}\boldsymbol{\epsilon}, 
\quad \boldsymbol{\epsilon} \sim \mathcal{N}(\mathbf{0}, \mathbf{I}).
\end{equation}
\begin{proposition}[Lower Bound via RGLM-Denoiser]\label{prop:latent_denoise}
Given the latent node representations $\mathbf{E}$ of a graph $\mathcal{G}$ obtained from a pre-trained GNN $g$, minimizing the diffusion noise prediction loss $\mathcal{L}_{\text{diff}}$ provides a variational lower bound of $\mathcal{I}(\mathcal{G}; \boldsymbol{s}^G)$:
\begin{equation}
\mathcal{I}(\mathcal{G};\boldsymbol{s}^{G})
\ \ge\ \mathcal{H}(\mathbf{E})-(\lambda_{l}\cdot\mathcal{L}_{\text{diff}}+\mathrm{C}).
\label{eq:denoise-mi-lb}
\end{equation}
where $\mathrm{C}$ is a constant independent of $\boldsymbol{s}^{G}$. Proof see Appendix~\ref{proof:latent_denoise}.
\end{proposition}

\section{Experiment}

\subsection{Experiment Setup}
\textbf{Datasets.} We conduct experiments on four widely-used datasets: Cora~\cite{yang2016revisiting}, Pubmed~\cite{yang2016revisiting} and OGBN-Arxiv~\cite{hu2020open} from citation networks, as well as Reddit~\cite{huang2024can} from a social network. 
Following~\cite{li2024glbench,chen2024llaga}, we select two widely adopted metrics: accuracy (Acc(↑)) and macro-F1 (F1(↑)).
More dataset details refer to Appendix~\ref{app: exp-dataset}.\\
\textbf{Baselines.}
We compare five categories of representative baselines:
(i) Traditional GNNs, GCN~\cite{DBLP:conf/iclr/KipfW17}, GAT~\cite{velivckovic2018graph} and GraphSAGE~\cite{hamilton2017inductive}; (ii) Graph transformer models, UniMP~\cite{shi2021masked}, DIFFormer~\cite{wu2023difformer} and NodeFormer~\cite{wu2022nodeformer}; (iii) Pretrain-finetune GNNs, GraphCL~\cite{you2020graph}, GRACE~\cite{zhu2020deep} and LGD~\cite{zhou2024unifying}; (iv) The GTextLLMs, GraphText-ICL~\cite{zhaographtext}; 
(v) The GTokenLLMs, GraphText-SFT~\cite{zhaographtext}, GraphGPT~\cite{tang2024graphgpt} and LLaGA~\cite{chen2024llaga}. 
We adopt Sentence-BERT~\cite{reimers2019sentence} to extract node features from the raw textual attributes, and all GTokenLLMs use Vicuna-7B-v1.5-16K~\cite{chiang2023vicuna} as the base LLM.
More baseline introduction and implementation details refer to Appendix~\ref{sec:exp-baseline} and Appendix~\ref{sec:exp-implementation}, respectively.

\subsection{Main Results}
\begin{figure*}[htbp]
     \centering
     \includegraphics[width=0.85\linewidth]{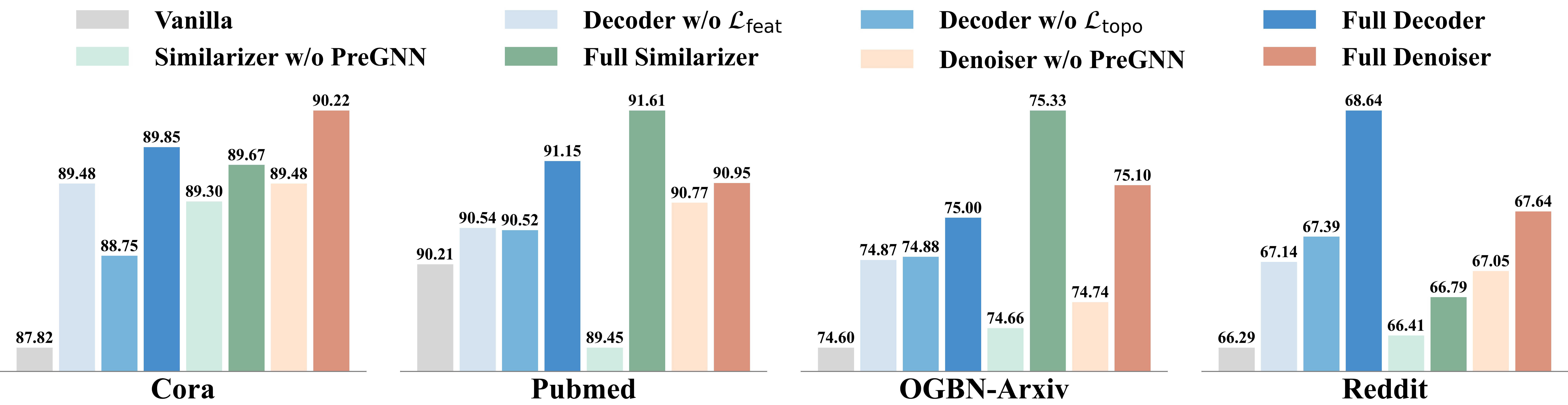}
     \vspace{-10pt}
     \caption{Ablation study on the node classification task across different datasets, where accuracy is reported for each variant.}
     \label{fig: ablation_node}
     \vspace{-5pt}
\end{figure*}
\begin{figure*}[htbp]
     \centering
     \includegraphics[width=0.85\linewidth]{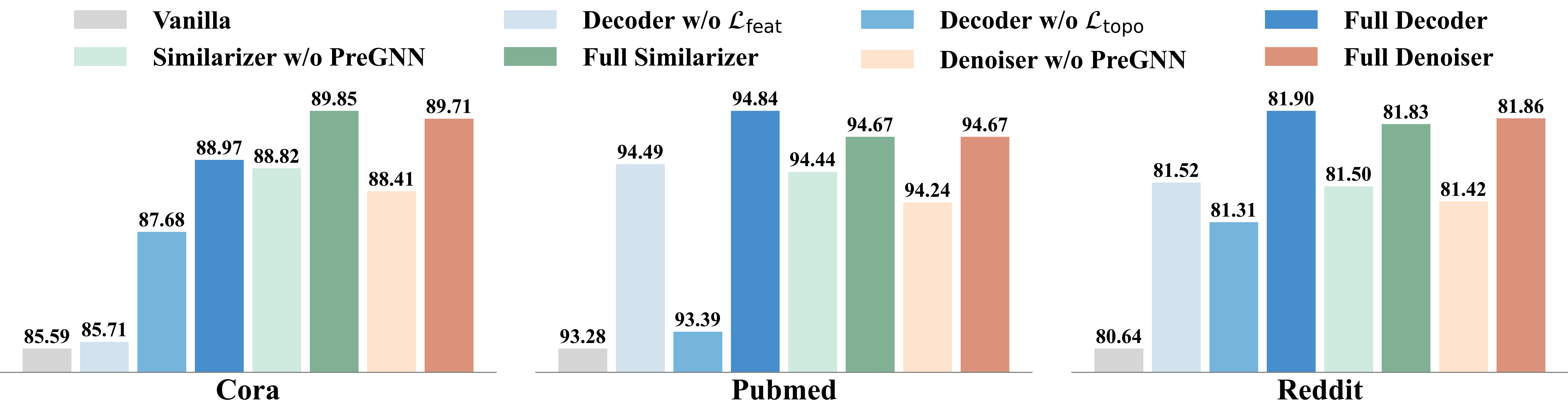}
     \vspace{-10pt}
     \caption{Ablation study on the link prediction task across different datasets, where accuracy is reported for each variant.}
     \label{fig: ablation_link}
     \vspace{-10pt}
\end{figure*}
\subsubsection{\textbf{Performance on Node Classification}} 
We mainly focus on the supervised node classification task. In Table~\ref{tab:main_node}, we find: 
(1) RGLM consistently outperforms all baselines. It highlights RGLM's emergent ability to tackle complex graph tasks beyond traditional GNNs. Compared to existing GTokenLLMs, RGLM introduces explicit graph supervision, which yields more effective alignment, leading to superior performance.
(2) Among the three RGLM variants, Decoder achieves the overall best performance. A possible reason is that reconstructing both the node features and the graph topology in the input space provides more comprehensive supervision for alignment. 
(3) Although we follow the pretraining framework of LGD for learning node latent representations, the performance of Similarizer/Denoiser is not upper-bounded by it. This is because graph reconstruction serves as a regularization term, aiming to enhance the attainable performance ceiling of the model.

\begin{table}[tp]
\caption{Link prediction results under supervised settings. Best are bolded, and runner-up are underlined.}\label{tab:main_link}
\vskip -0.15in
\setlength{\extrarowheight}{1.5pt}
\resizebox{\linewidth}{!}{
\Large
\begin{tabular}{>{\centering\arraybackslash}m{1.4cm}|
                >{\centering\arraybackslash}m{2.65cm}|
                cc|cc|cc}
\hline
\multirow{2}{*}{Method}     & Dataset          & \multicolumn{2}{c|}{Cora}       & \multicolumn{2}{c|}{Pubmed}     & \multicolumn{2}{c}{Reddit}      \\ \cline{2-8} 
                            & Metric           & Acc            & F1             & Acc            & F1             & Acc            & F1             \\ \hline
\multirow{6}{*}{\Large{Baselines}}  
                            & \large{GCN}              & 84.26          & 83.91          & 90.00          & 89.98          & 74.38          & 73.85          \\
                            & \large{UniMP}            & 71.62          & 70.92          & 80.76          & 80.54          & 73.09          & 72.53          \\
                            & \large{GRACE}            & 70.44          & 69.19          & 81.74          & 81.56          & 61.68          & 60.88          \\  
                            & \large{GraphGPT}         & 59.85          & 59.85          & 73.60          & 73.39          & 56.69          & 27.19          \\
                            & \large{LLaGA}            & 82.21          & 79.66          & 91.78          & 91.88          & 79.73          & 76.60          \\ \hline
\multirow{3}{*}{\Large{Ours}}     
                            & \large{RGLM-Decoder}     & 88.97          & 87.96          & \textbf{94.84} & \textbf{94.91} & \textbf{81.90} & {\ul 77.97}    \\
                            & \large{RGLM-Similarizer} & \textbf{89.85} & {\ul 88.82}    & 94.67          & 94.72          & 81.83          & \textbf{78.10} \\
                            & \large{RGLM-Denoiser}    & {\ul 89.71}    & \textbf{88.92} & {\ul 94.67}    & {\ul 94.74}    & {\ul 81.86}    & 77.93          \\ \hline
\end{tabular}
}
\vskip -0.175in
\end{table}

\subsubsection{\textbf{Performance on Link Prediction}}
We also evaluate RGLM on the supervised link prediction task. Following works~\cite{chen2024llaga,tang2024graphgpt}, we randomly select node pairs from the node classification training set for training and from the test set for evaluation, ensuring the edge-level training sets match the size of the node-level. As shown in Table~\ref{tab:main_link}, RGLM consistently outperforms all baselines in the link prediction task. Compared to GraphGPT and LLaGA, which solely leverage text supervision, RGLM explicitly incorporates graph supervision into the alignment process, resulting in an accuracy improvement of 2.1 to 30 points. This highlights the critical role of explicit graph supervision in supporting graph–text alignment.







\begin{table}[tp]
\caption{Node classification results under multi-dataset training. Best are bolded, and runner-up are underlined.}\label{tab:multi-dataset}
\vskip -0.15in
\Large
\setlength{\extrarowheight}{-0.6pt}
\resizebox{\linewidth}{!}{
\begin{tabular}{>{\centering\arraybackslash}m{1.4cm}|
                >{\centering\arraybackslash}m{2.65cm}|
                cc|cc}
\hline
\multirow{3}{*}{Method}     & Training set     & \multicolumn{4}{c}{Cora+Pubmed}                                                        \\ \cline{2-6} 
                            & Test set         & \multicolumn{2}{c|}{Cora}                            & \multicolumn{2}{c}{Pubmed}      \\ \cline{2-6} 
                            & Metric           & Acc            & \multicolumn{1}{c|}{F1}             & Acc            & F1             \\ \hline
\multirow{8}{*}{\Large{Baselines}}       
                            & \large{GCN}              & 85.61          & \multicolumn{1}{c|}{84.16}          & 85.07          & 84.29          \\
                            & \large{GraphSAGE}        & 86.35          & \multicolumn{1}{c|}{86.21}          & 89.27          & 89.12          \\
                            & \large{GAT}              & 83.03          & \multicolumn{1}{c|}{82.29}          & 86.51          & 85.99          \\
                            & \large{UniMP}            & 86.35          & \multicolumn{1}{c|}{85.21}          & 89.20          & 89.14          \\
                            & \large{NodeFormer}       & 87.30          & \multicolumn{1}{c|}{85.44}          & 88.44          & 88.22          \\
                            & \large{GRACE}            & 88.56          & \multicolumn{1}{c|}{87.55}          & 86.21          & 85.92          \\
                            & \large{GraphGPT}         & 81.59          & \multicolumn{1}{c|}{79.83}          & 78.21          & 77.45          \\
                            & \large{LLaGA}            & 88.53          & \multicolumn{1}{c|}{87.28}          & 89.60          & 89.46          \\ \hline
\multirow{3}{*}{\Large{Ours}}     
                            & \large{RGLM-Decoder}     & {\ul 89.67}    & \multicolumn{1}{c|}{{\ul 88.14}}    & {\ul 90.37}    & 90.32          \\
                            & \large{RGLM-Similarizer} & 88.93          & \multicolumn{1}{c|}{87.49}          & \textbf{91.18} & \textbf{91.07} \\
                            & \large{RGLM-Denoiser}    & \textbf{90.04} & \multicolumn{1}{c|}{\textbf{88.94}} & 90.34          & {\ul 90.36}    \\ \hline
\end{tabular}
}
\vskip -0.2in
\end{table}

\subsubsection{\textbf{Multi-dataset Generalization Ability}}
Typically, training a model on multiple datasets may lead to catastrophic forgetting~\cite{kirkpatrick2017overcoming}, resulting in degraded performance. Here, we examine the multi-dataset generalization capability of RGLM across multiple datasets. In our study, the models are first trained on node classification tasks using both the Cora and Pubmed, and then evaluated on each dataset. As reported in Table~\ref{tab:multi-dataset}, RGLM remains robust when trained across different datasets, achieving superior performance and demonstrating its capability in cross-dataset learning.

\subsubsection{\textbf{Zero-Shot Ability}}
We investigate the zero-shot capability of RGLM, where the model is trained on a source dataset and evaluated on a target dataset without any access to the target data during training. 
Following LLaGA~\cite{chen2024llaga}, we train the model on the node description and link prediction tasks with OGBN-Arxiv, and evaluate it on node classification and link prediction using Cora and Pubmed, respectively.
As shown in Figure~\ref{fig:zero_shot_acc}, the three RGLM variants have robust zero-shot capabilities across both tasks. By introducing the additional graph reconstruction tasks, the knowledge can be more effectively transferred across different scenarios.
\begin{figure}[tp]
    \centering
    \includegraphics[width=1.0\linewidth]{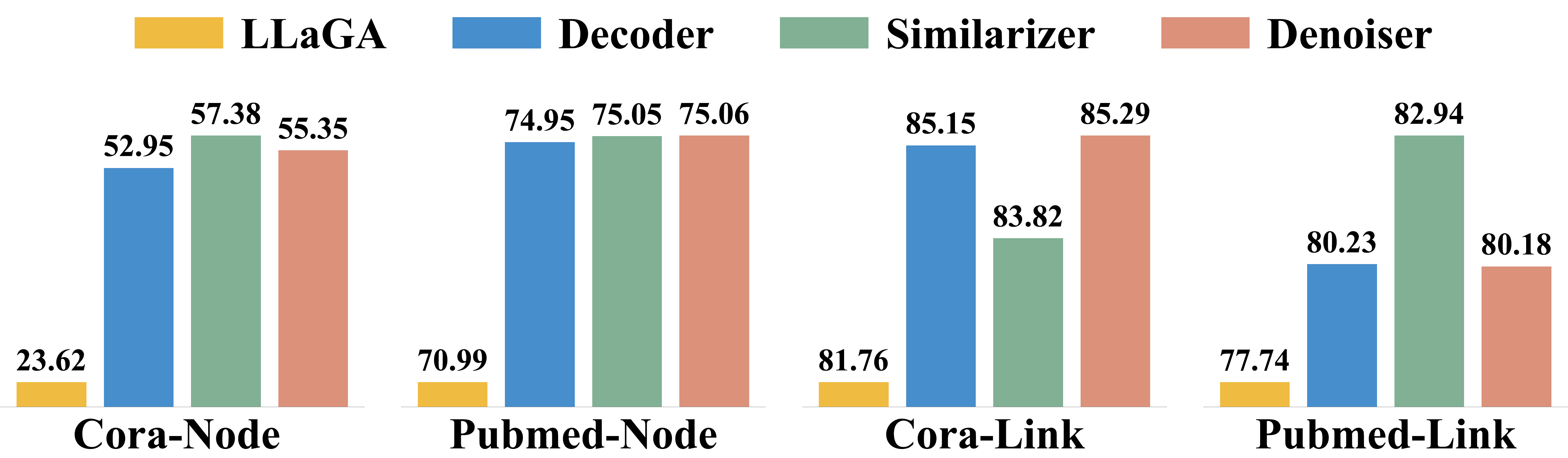}
    \vspace{-20pt}
    \caption{Zero-shot accuracy on node classification and link prediction tasks, where models are trained on the OGBN-Arxiv and evaluated on unseen Cora and Pubmed.}
    \vspace{-15pt}
    \label{fig:zero_shot_acc}
\end{figure}





\subsection{Model Analysis}
\subsubsection{\textbf{Ablation Study}}
We first evaluate the importance of key components through an ablation study. Figures~\ref{fig: ablation_node} and~\ref{fig: ablation_link} present the ablation results of the three RGLM variants on node classification and link prediction, respectively. Specifically, ``Vanilla'' denotes the setting without any reconstruction loss, where LLMs are tuned solely through text supervision. ``w/o $\mathcal{L}_{\text{feat}}$'' and ``w/o $\mathcal{L}_{\text{topo}}$'' correspond to removing the feature decoder and the structure decoder, respectively. ``w/o PreGNN'' refers to removing the pre-trained GNN, which is equivalent to reconstructing the input features while ignoring structural information on the graph. Based on the results, we observe:
(1) Explicitly incorporating graph supervision into the instruction tuning process consistently improves performance compared with pure text supervision, demonstrating the effectiveness of RGLM.
(2) For Decoder, introducing both feature reconstruction and structure reconstruction consistently improves performance, with the best results achieved when both are applied together.
(3) For Similarizer and Denoiser, even without the pre-trained GNN, RGLM still yields performance gains in most cases, which is because they approximately degenerate into input feature reconstruction.
\begin{figure}[tp]
    \centering
    \vspace{-7pt}
    \includegraphics[width=1.0\linewidth]{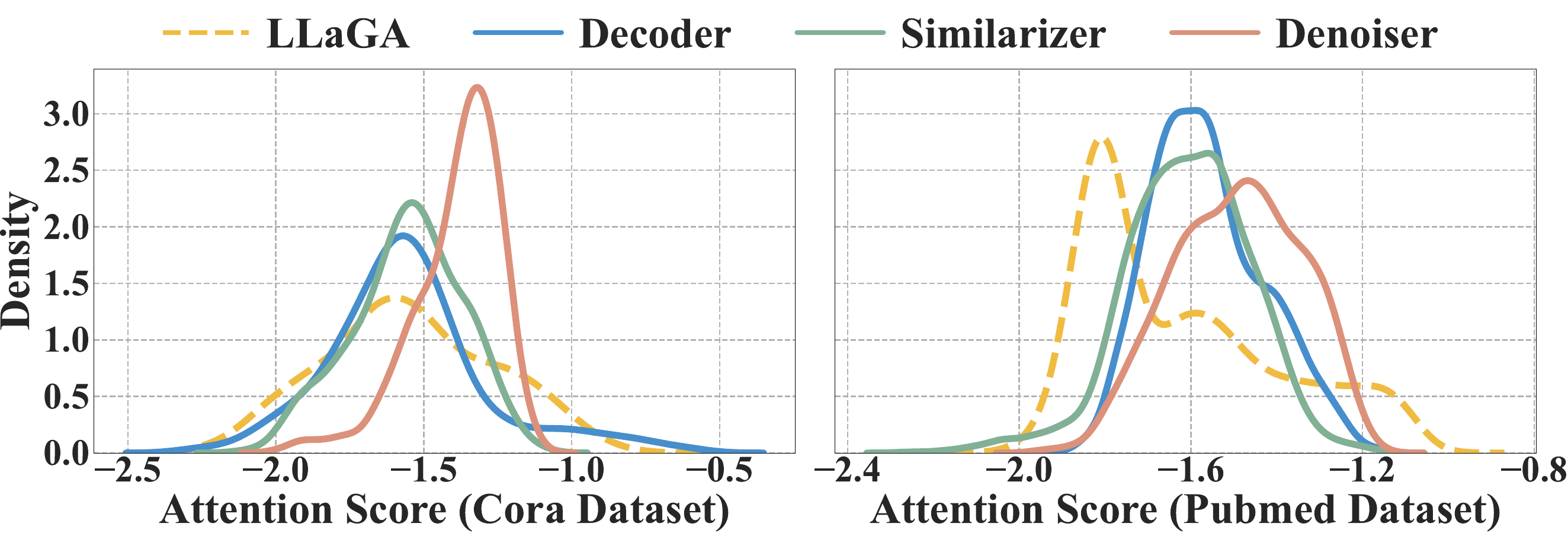}
    \vspace{-20pt}
    \caption{Attention distribution over graph tokens for different methods, where the attention scores are log-scaled.}
    \vspace{5pt}
    \label{fig:attn-dist}
\end{figure}
\subsubsection{\textbf{Attention Analysis}}
We compute the attention scores of the last token with respect to all graph tokens, excluding the placeholder tokens, on the test set of Cora and Pubmed. Figure~\ref{fig:attn-dist} illustrates the attention distribution over graph tokens for LLaGA and RGLM. 
The results reveal that RGLM yields significantly higher attention scores than LLaGA, indicating that incorporating the graph reconstruction objective effectively guides the LLM to focus on input graphs and enhances its ability to comprehend graphs.

\begin{figure}[tp]
    \centering
    \includegraphics[width=1.0\linewidth]{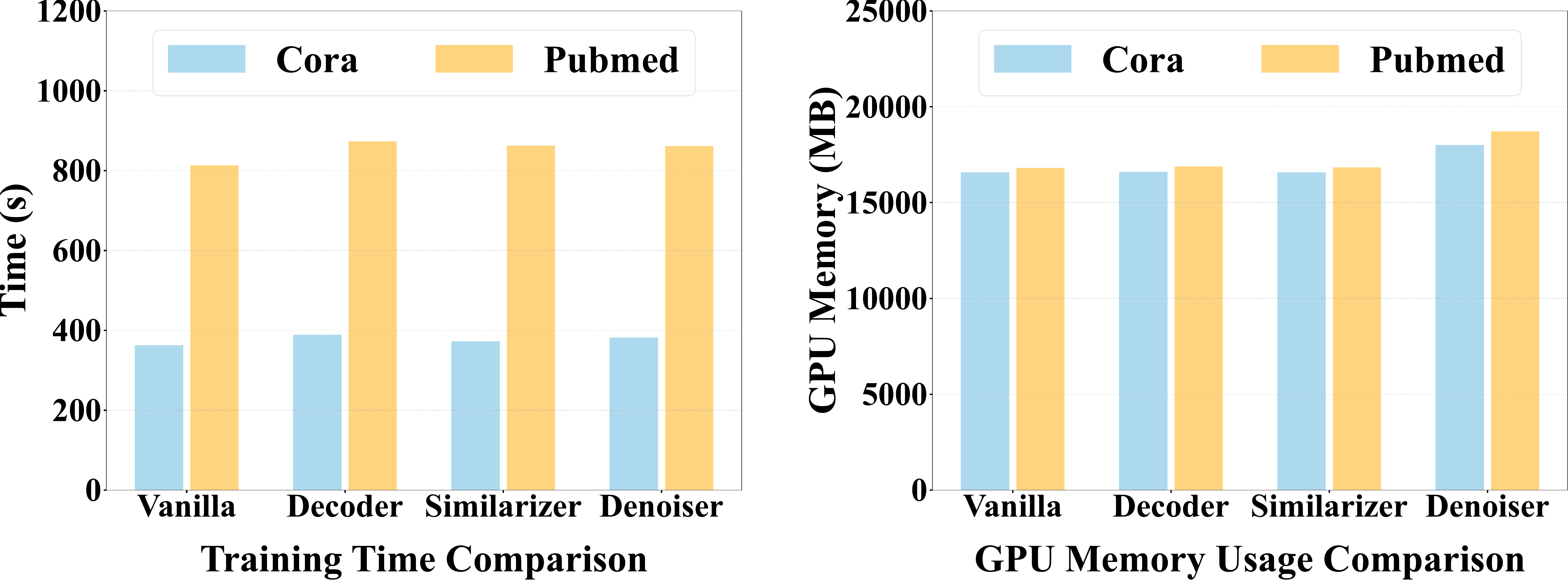}
    \vspace{-18pt}
    \caption{Time complexity (left) and resource cost (right) analysis of different RGLM variants.}
    \vspace{-13pt}
    \label{fig:complexity}
\end{figure}
\subsubsection{\textbf{Compatibility Analysis}}
We separately adopt LLaMA3-8B~\cite{touvron2023llama} and Vicuna-13B~\cite{chiang2023vicuna} as base LLMs to evaluate RGLM's compatibility.
We report the accuracy of various LLMs under two settings: without the reconstruction loss ("Vanilla") and with the proposed RGLM. As shown in Figure~\ref{fig:compatibility}, introducing reconstructive graph loss consistently improves the performance of various LLMs compared with pure text supervision, demonstrating that RGLM is compatible with diverse LLMs and serves as a universal pipeline.

\subsubsection{\textbf{Complexity Analysis}}
We further evaluate the time and space complexity of RGLM. 
Since the graph reconstruction loss does not affect the inference stage of the LLM, here we focus on the instruction tuning phase. Figure~\ref{fig:complexity} reports the training time and GPU memory overhead of the three variants on the Cora and Pubmed datasets. 
The results indicate that the reconstruction process introduces only a negligible increase in training time and resources, while the performance gains clearly outweigh this minor overhead.

\subsubsection{\textbf{Hyper-parameter Analysis}}
Finally, we conduct sensitivity analysis on the reconstruction loss coefficients: $\lambda_f$ and $\lambda_s$ for Decoder, $\lambda_{l}$ for Similarizer and Denoiser. As shown in  Figure~\ref{fig:hyper_s_f_decoder} and Figure~\ref{fig:hyper_l_latent}, RGLM's performance increases at first and then decreases as the reconstruction loss weight grows. This trend stems from the role of reconstruction objectives as regularization terms: when the weight is too small, the explicit graph supervision becomes insufficient, leading to under-utilization of graph information; when the weight is too large, the reconstruction loss dominates training and interferes with language modeling, degrading the performance. 
\begin{figure}[tp]
    \centering
    \includegraphics[width=1.0\linewidth]{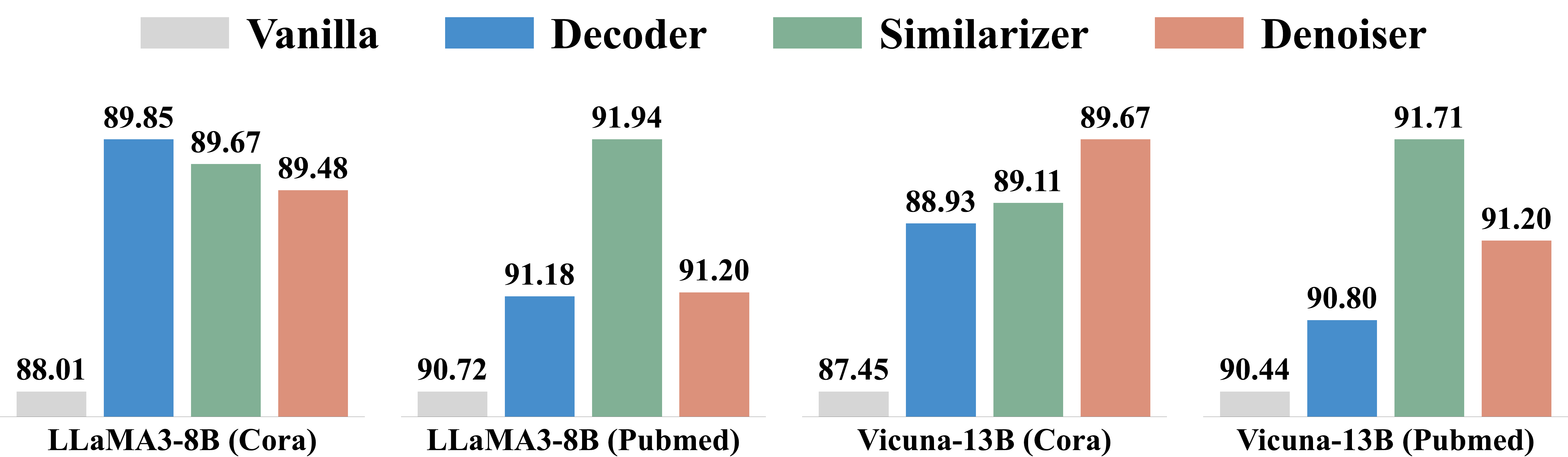}
    \vspace{-15pt}
    \caption{Compatibility analysis of different RGLM variants on base LLMs, where accuracy is reported for each variant.}
    \vspace{5pt}
    \label{fig:compatibility}
\end{figure}

\begin{figure}[tp]
    \centering
    \vspace{-11pt}
    \begin{subfigure}[t]{0.2\textwidth}
        \centering
        \includegraphics[width=\linewidth]{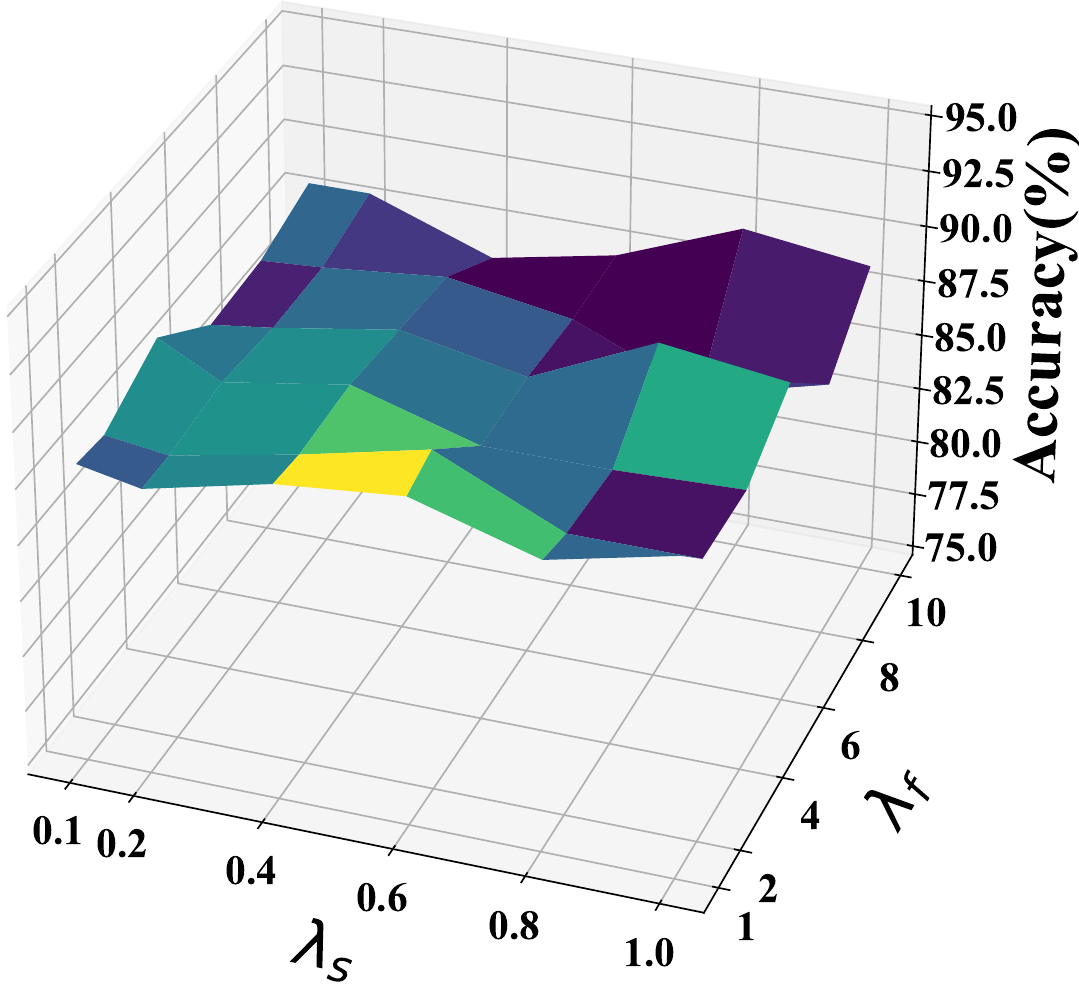}
        \vspace{-16pt}
        \caption{Cora.}
    \end{subfigure}
    \hfill
    \vspace{-11pt}
    \begin{subfigure}[t]{0.2\textwidth}
        \centering
        \includegraphics[width=\linewidth]{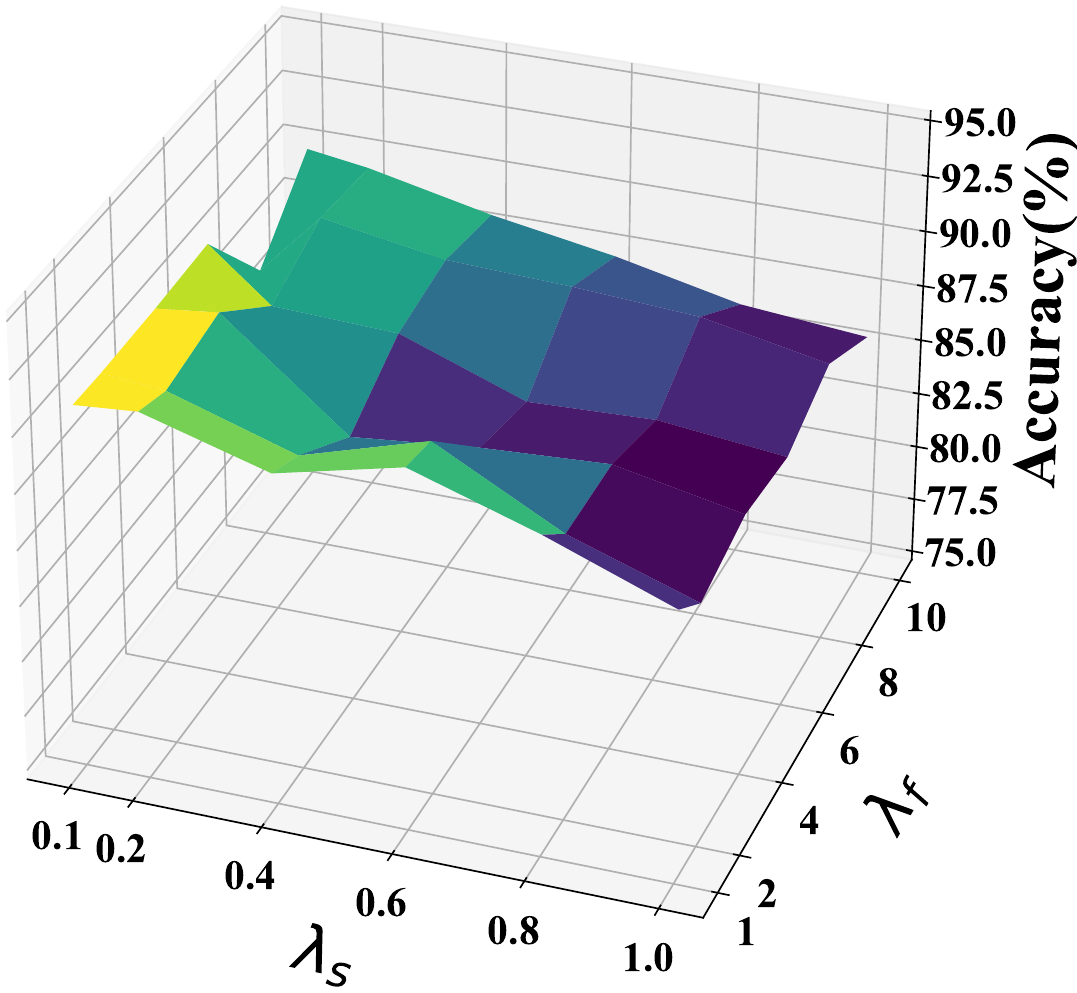}
        \vspace{-16pt}
        \caption{Pubmed.}
    \end{subfigure}
    \caption{Analysis of the hyper-parameter $\lambda_{s}$ and $\lambda_{f}$ in the RGLM-Decoder variant.}
    \vspace{-5pt}
    \label{fig:hyper_s_f_decoder}
\end{figure}

\begin{figure}[tp]
    \centering
    \begin{subfigure}[t]{0.23\textwidth}
        \centering
        \includegraphics[width=\linewidth]{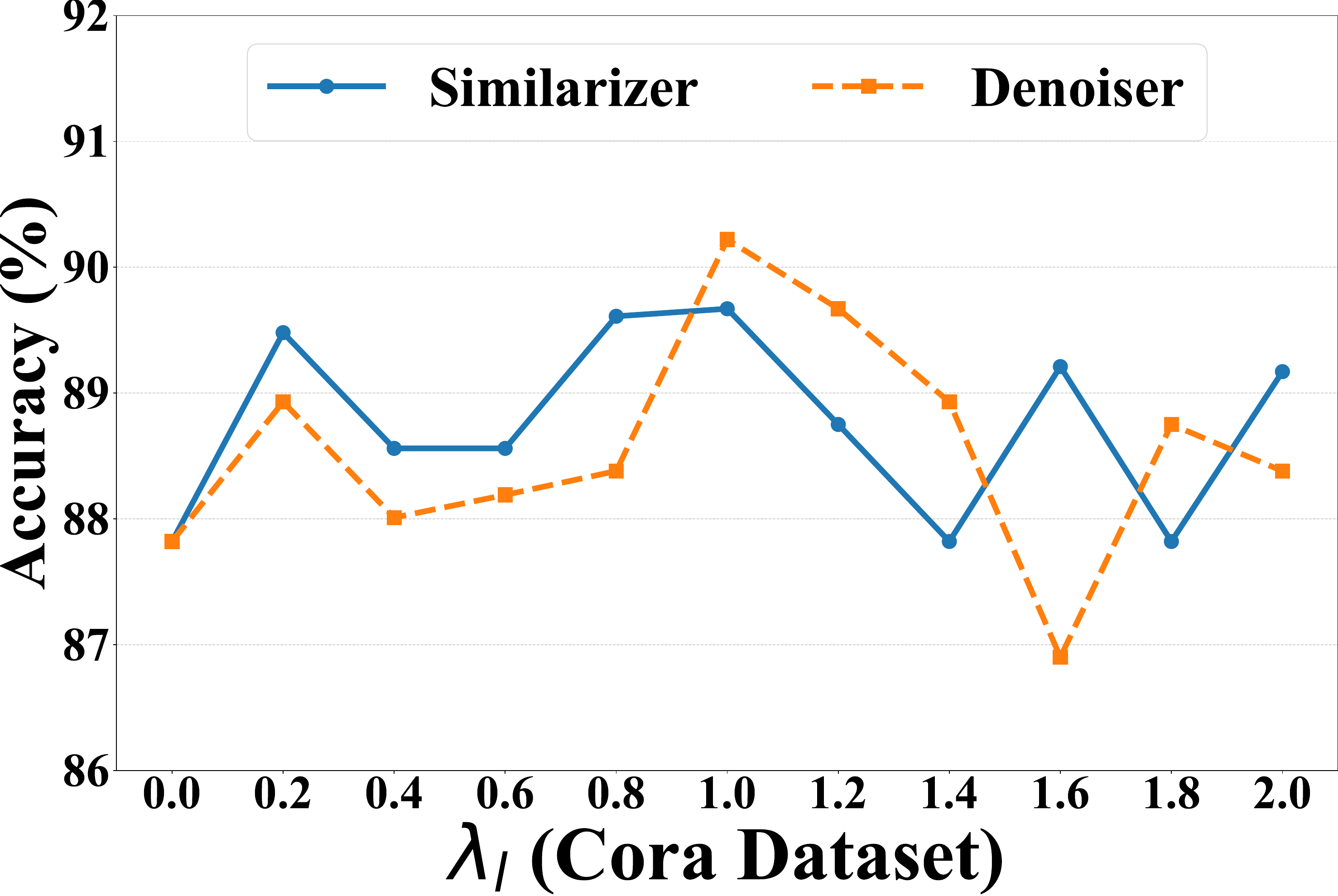}
    \end{subfigure}
    \hfill
    \begin{subfigure}[t]{0.23\textwidth}
        \centering
        \includegraphics[width=\linewidth]{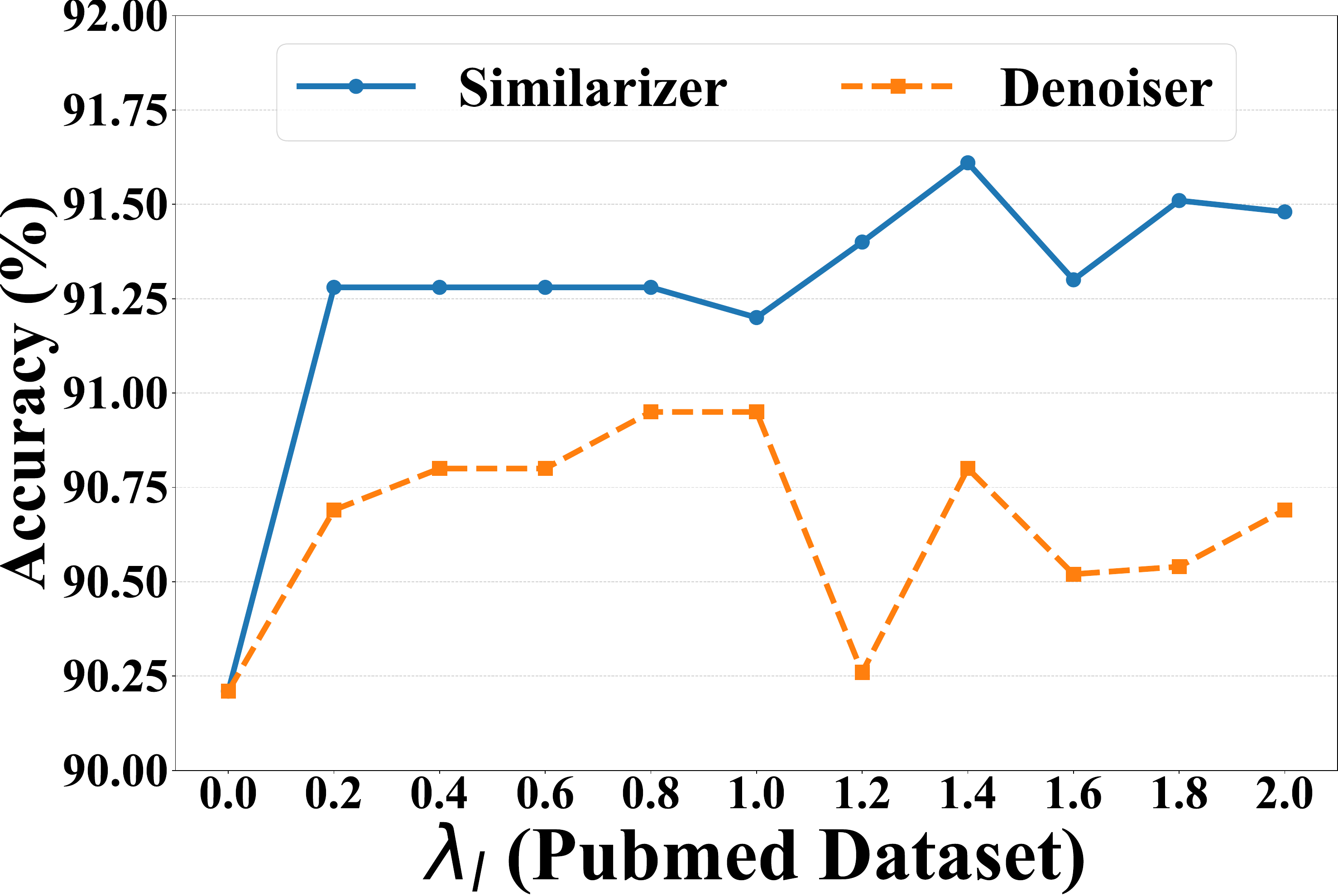}
    \end{subfigure}
    \vspace{-10pt}
    \caption{Analysis of the hyper-parameter $\lambda_{l}$ in the RGLM-Similarizer and RGLM-Denoiser variants.}
    \vspace{-10pt}
    \label{fig:hyper_l_latent}
\end{figure}

\section{Conclusion}
In this work, we theoretically reveal the limitations of existing GTokenLLMs that rely solely on text supervision from language instructions, resulting in implicit graph–text alignment. To overcome it, we prove that the alignment objective is upper-bounded by the mutual information between the input graphs and their hidden representations in the LLM, and further propose a reconstructive graph instruction tuning pipeline, RGLM, to improve this upper bound. The RGLM is embodied through three distinct variants: Decoder from the input space; Similarizer and Denoiser from the latent space. Extensive experiments validate RGLM's effectiveness.

\begin{acks}
This work is supported in part by the National Natural Science Foundation of China (No. 62550138, 62192784, 62572064, 62472329, 62322203, 62172052), and the BUPT Excellent Ph.D. Students Foundation (No. CX20251005).
\end{acks}
\clearpage


\bibliographystyle{ACM-Reference-Format}
\balance
\bibliography{main}


\begin{thebibliography}{62}


\ifx \showCODEN    \undefined \def \showCODEN     #1{\unskip}     \fi
\ifx \showISBNx    \undefined \def \showISBNx     #1{\unskip}     \fi
\ifx \showISBNxiii \undefined \def \showISBNxiii  #1{\unskip}     \fi
\ifx \showISSN     \undefined \def \showISSN      #1{\unskip}     \fi
\ifx \showLCCN     \undefined \def \showLCCN      #1{\unskip}     \fi
\ifx \shownote     \undefined \def \shownote      #1{#1}          \fi
\ifx \showarticletitle \undefined \def \showarticletitle #1{#1}   \fi
\ifx \showURL      \undefined \def \showURL       {\relax}        \fi
\providecommand\bibfield[2]{#2}
\providecommand\bibinfo[2]{#2}
\providecommand\natexlab[1]{#1}
\providecommand\showeprint[2][]{arXiv:#2}

\bibitem[Achiam et~al\mbox{.}(2023)]%
        {achiam2023gpt}
\bibfield{author}{\bibinfo{person}{Josh Achiam}, \bibinfo{person}{Steven Adler}, \bibinfo{person}{Sandhini Agarwal}, {et~al\mbox{.}}} \bibinfo{year}{2023}\natexlab{}.
\newblock \showarticletitle{Gpt-4 technical report}.
\newblock \bibinfo{journal}{\emph{arXiv preprint arXiv:2303.08774}} (\bibinfo{year}{2023}).
\newblock


\bibitem[Chen et~al\mbox{.}(2025)]%
        {chen2025sustainability}
\bibfield{author}{\bibinfo{person}{Jinwen Chen}, \bibinfo{person}{Hao Miao}, \bibinfo{person}{Dazhuo Qiu}, {et~al\mbox{.}}} \bibinfo{year}{2025}\natexlab{}.
\newblock \showarticletitle{Sustainability-Oriented Task Recommendation in Spatial Crowdsourcing}. In \bibinfo{booktitle}{\emph{ICDE}}.
\newblock


\bibitem[Chen et~al\mbox{.}(2024b)]%
        {chen2024llaga}
\bibfield{author}{\bibinfo{person}{Runjin Chen}, \bibinfo{person}{Tong Zhao}, \bibinfo{person}{Ajay~Kumar Jaiswal}, {et~al\mbox{.}}} \bibinfo{year}{2024}\natexlab{b}.
\newblock \showarticletitle{LLaGA: Large Language and Graph Assistant}. In \bibinfo{booktitle}{\emph{International Conference on Machine Learning}}. PMLR.
\newblock


\bibitem[Chen et~al\mbox{.}(2016)]%
        {chen2016infogan}
\bibfield{author}{\bibinfo{person}{Xi Chen}, \bibinfo{person}{Yan Duan}, \bibinfo{person}{Rein Houthooft}, {et~al\mbox{.}}} \bibinfo{year}{2016}\natexlab{}.
\newblock \showarticletitle{Infogan: Interpretable representation learning by information maximizing generative adversarial nets}.
\newblock \bibinfo{journal}{\emph{Advances in neural information processing systems}}  \bibinfo{volume}{29} (\bibinfo{year}{2016}).
\newblock


\bibitem[Chen et~al\mbox{.}(2024a)]%
        {chen2024exploring}
\bibfield{author}{\bibinfo{person}{Zhikai Chen}, \bibinfo{person}{Haitao Mao}, {et~al\mbox{.}}} \bibinfo{year}{2024}\natexlab{a}.
\newblock \showarticletitle{Exploring the potential of large language models (llms) in learning on graphs}.
\newblock \bibinfo{journal}{\emph{ACM SIGKDD Explorations Newsletter}} (\bibinfo{year}{2024}).
\newblock


\bibitem[Chiang et~al\mbox{.}(2023)]%
        {chiang2023vicuna}
\bibfield{author}{\bibinfo{person}{Wei-Lin Chiang}, \bibinfo{person}{Zhuohan Li}, \bibinfo{person}{Ziqing Lin}, {et~al\mbox{.}}} \bibinfo{year}{2023}\natexlab{}.
\newblock \showarticletitle{Vicuna: An open-source chatbot impressing gpt-4 with 90\%* chatgpt quality}.
\newblock \bibinfo{journal}{\emph{See https://vicuna. lmsys. org (accessed 14 April 2023)}} \bibinfo{volume}{2}, \bibinfo{number}{3} (\bibinfo{year}{2023}), \bibinfo{pages}{6}.
\newblock


\bibitem[Devlin et~al\mbox{.}(2019)]%
        {devlin2019bert}
\bibfield{author}{\bibinfo{person}{Jacob Devlin}, \bibinfo{person}{Ming-Wei Chang}, {et~al\mbox{.}}} \bibinfo{year}{2019}\natexlab{}.
\newblock \showarticletitle{BERT: Pre-training of Deep Bidirectional Transformers for Language Understanding}. In \bibinfo{booktitle}{\emph{North American Chapter of the Association for Computational Linguistics}}.
\newblock


\bibitem[Fatemi et~al\mbox{.}(2024)]%
        {fatemitalk}
\bibfield{author}{\bibinfo{person}{Bahare Fatemi}, \bibinfo{person}{Jonathan Halcrow}, {and} \bibinfo{person}{Bryan Perozzi}.} \bibinfo{year}{2024}\natexlab{}.
\newblock \showarticletitle{Talk like a Graph: Encoding Graphs for Large Language Models}. In \bibinfo{booktitle}{\emph{The Twelfth International Conference on Learning Representations}}.
\newblock


\bibitem[Fey and Lenssen(2019)]%
        {fey2019fast}
\bibfield{author}{\bibinfo{person}{Matthias Fey} {and} \bibinfo{person}{Jan~Eric Lenssen}.} \bibinfo{year}{2019}\natexlab{}.
\newblock \showarticletitle{Fast graph representation learning with PyTorch Geometric}.
\newblock \bibinfo{journal}{\emph{arXiv preprint arXiv:1903.02428}} (\bibinfo{year}{2019}).
\newblock


\bibitem[Guo et~al\mbox{.}(2023)]%
        {guo2023gpt4graph}
\bibfield{author}{\bibinfo{person}{Jiayan Guo}, \bibinfo{person}{Lun Du}, \bibinfo{person}{Hengyu Liu}, {et~al\mbox{.}}} \bibinfo{year}{2023}\natexlab{}.
\newblock \showarticletitle{Gpt4graph: Can large language models understand graph structured data? an empirical evaluation and benchmarking}.
\newblock \bibinfo{journal}{\emph{arXiv preprint arXiv:2305.15066}} (\bibinfo{year}{2023}).
\newblock


\bibitem[Guo et~al\mbox{.}(2024)]%
        {guo2024data}
\bibfield{author}{\bibinfo{person}{Yuxin Guo}, \bibinfo{person}{Deyu Bo}, \bibinfo{person}{Cheng Yang}, \bibinfo{person}{Zhiyuan Lu}, \bibinfo{person}{Zhongjian Zhang}, \bibinfo{person}{Jixi Liu}, \bibinfo{person}{Yufei Peng}, {and} \bibinfo{person}{Chuan Shi}.} \bibinfo{year}{2024}\natexlab{}.
\newblock \showarticletitle{Data-centric graph learning: A survey}.
\newblock \bibinfo{journal}{\emph{IEEE Transactions on Big Data}} \bibinfo{volume}{11}, \bibinfo{number}{1} (\bibinfo{year}{2024}), \bibinfo{pages}{1--20}.
\newblock


\bibitem[Guo et~al\mbox{.}({[n.\,d.]})]%
        {guograph}
\bibfield{author}{\bibinfo{person}{Zeyuan Guo}, \bibinfo{person}{Enmao Diao}, \bibinfo{person}{Cheng Yang}, {and} \bibinfo{person}{Chuan Shi}.} \bibinfo{year}{[n.\,d.]}\natexlab{}.
\newblock \showarticletitle{Graph Tokenization for Bridging Graphs and Transformers}. In \bibinfo{booktitle}{\emph{The Fourteenth International Conference on Learning Representations}}.
\newblock


\bibitem[Hamilton et~al\mbox{.}(2017)]%
        {hamilton2017inductive}
\bibfield{author}{\bibinfo{person}{Will Hamilton}, \bibinfo{person}{Zhitao Ying}, {and} \bibinfo{person}{Jure Leskovec}.} \bibinfo{year}{2017}\natexlab{}.
\newblock \showarticletitle{Inductive representation learning on large graphs}.
\newblock \bibinfo{journal}{\emph{Advances in neural information processing systems}} (\bibinfo{year}{2017}).
\newblock


\bibitem[He et~al\mbox{.}(2025)]%
        {he2025unigraph}
\bibfield{author}{\bibinfo{person}{Yufei He}, \bibinfo{person}{Yuan Sui}, \bibinfo{person}{Xiaoxin He}, {et~al\mbox{.}}} \bibinfo{year}{2025}\natexlab{}.
\newblock \showarticletitle{Unigraph: Learning a unified cross-domain foundation model for text-attributed graphs}. In \bibinfo{booktitle}{\emph{Proceedings of the 31st ACM SIGKDD Conference on Knowledge Discovery and Data Mining V. 1}}.
\newblock


\bibitem[Hjelm et~al\mbox{.}(2019)]%
        {hjelmlearning}
\bibfield{author}{\bibinfo{person}{R~Devon Hjelm}, \bibinfo{person}{Alex Fedorov}, \bibinfo{person}{Samuel Lavoie-Marchildon}, {et~al\mbox{.}}} \bibinfo{year}{2019}\natexlab{}.
\newblock \showarticletitle{Learning deep representations by mutual information estimation and maximization}. In \bibinfo{booktitle}{\emph{International Conference on Learning Representations}}.
\newblock


\bibitem[Hu et~al\mbox{.}(2022)]%
        {hu2022lora}
\bibfield{author}{\bibinfo{person}{Edward~J Hu}, \bibinfo{person}{Yelong Shen}, \bibinfo{person}{Phillip Wallis}, {et~al\mbox{.}}} \bibinfo{year}{2022}\natexlab{}.
\newblock \showarticletitle{Lora: Low-rank adaptation of large language models.}
\newblock \bibinfo{journal}{\emph{ICLR}} \bibinfo{volume}{1}, \bibinfo{number}{2} (\bibinfo{year}{2022}), \bibinfo{pages}{3}.
\newblock


\bibitem[Hu et~al\mbox{.}(2020)]%
        {hu2020open}
\bibfield{author}{\bibinfo{person}{Weihua Hu}, \bibinfo{person}{Matthias Fey}, \bibinfo{person}{Marinka Zitnik}, {et~al\mbox{.}}} \bibinfo{year}{2020}\natexlab{}.
\newblock \showarticletitle{Open graph benchmark: Datasets for machine learning on graphs}.
\newblock \bibinfo{journal}{\emph{Advances in neural information processing systems}} (\bibinfo{year}{2020}).
\newblock


\bibitem[Huang et~al\mbox{.}(2024)]%
        {huang2024can}
\bibfield{author}{\bibinfo{person}{Xuanwen Huang}, \bibinfo{person}{Kaiqiao Han}, \bibinfo{person}{Yang Yang}, {et~al\mbox{.}}} \bibinfo{year}{2024}\natexlab{}.
\newblock \showarticletitle{Can gnn be good adapter for llms?}. In \bibinfo{booktitle}{\emph{Proceedings of the ACM Web Conference 2024}}.
\newblock


\bibitem[Jin et~al\mbox{.}(2024)]%
        {jin2024large}
\bibfield{author}{\bibinfo{person}{Bowen Jin}, \bibinfo{person}{Gang Liu}, {et~al\mbox{.}}} \bibinfo{year}{2024}\natexlab{}.
\newblock \showarticletitle{Large language models on graphs: A comprehensive survey}.
\newblock \bibinfo{journal}{\emph{IEEE Transactions on Knowledge and Data Engineering}} (\bibinfo{year}{2024}).
\newblock


\bibitem[Kipf and Welling(2017)]%
        {DBLP:conf/iclr/KipfW17}
\bibfield{author}{\bibinfo{person}{Thomas~N. Kipf} {and} \bibinfo{person}{Max Welling}.} \bibinfo{year}{2017}\natexlab{}.
\newblock \showarticletitle{Semi-Supervised Classification with Graph Convolutional Networks}. In \bibinfo{booktitle}{\emph{International Conference on Learning Representations}}.
\newblock


\bibitem[Kirkpatrick et~al\mbox{.}(2017)]%
        {kirkpatrick2017overcoming}
\bibfield{author}{\bibinfo{person}{James Kirkpatrick}, \bibinfo{person}{Razvan Pascanu}, \bibinfo{person}{Neil Rabinowitz}, {et~al\mbox{.}}} \bibinfo{year}{2017}\natexlab{}.
\newblock \showarticletitle{Overcoming catastrophic forgetting in neural networks}.
\newblock \bibinfo{journal}{\emph{Proceedings of the national academy of sciences}} (\bibinfo{year}{2017}).
\newblock


\bibitem[Kong et~al\mbox{.}(2025)]%
        {konggofa}
\bibfield{author}{\bibinfo{person}{Lecheng Kong}, \bibinfo{person}{Jiarui Feng}, \bibinfo{person}{Hao Liu}, {et~al\mbox{.}}} \bibinfo{year}{2025}\natexlab{}.
\newblock \showarticletitle{GOFA: A Generative One-For-All Model for Joint Graph Language Modeling}. In \bibinfo{booktitle}{\emph{The Thirteenth International Conference on Learning Representations}}.
\newblock


\bibitem[Li et~al\mbox{.}(2020)]%
        {li2020deepergcn}
\bibfield{author}{\bibinfo{person}{Guohao Li}, \bibinfo{person}{Chenxin Xiong}, \bibinfo{person}{Ali Thabet}, {et~al\mbox{.}}} \bibinfo{year}{2020}\natexlab{}.
\newblock \showarticletitle{Deepergcn: All you need to train deeper gcns}.
\newblock \bibinfo{journal}{\emph{arXiv preprint arXiv:2006.07739}} (\bibinfo{year}{2020}).
\newblock


\bibitem[Li et~al\mbox{.}(2024)]%
        {li2024glbench}
\bibfield{author}{\bibinfo{person}{Yuhan Li}, \bibinfo{person}{Peisong Wang}, \bibinfo{person}{Xiao Zhu}, {et~al\mbox{.}}} \bibinfo{year}{2024}\natexlab{}.
\newblock \showarticletitle{Glbench: A comprehensive benchmark for graph with large language models}.
\newblock \bibinfo{journal}{\emph{Advances in Neural Information Processing Systems}}  \bibinfo{volume}{37} (\bibinfo{year}{2024}), \bibinfo{pages}{42349--42368}.
\newblock


\bibitem[Liu et~al\mbox{.}(2025b)]%
        {liu2025graph}
\bibfield{author}{\bibinfo{person}{Jiawei Liu}, \bibinfo{person}{Cheng Yang}, \bibinfo{person}{Zhiyuan Lu}, {et~al\mbox{.}}} \bibinfo{year}{2025}\natexlab{b}.
\newblock \showarticletitle{Graph foundation models: Concepts, opportunities and challenges}.
\newblock \bibinfo{journal}{\emph{IEEE Transactions on Pattern Analysis and Machine Intelligence}} (\bibinfo{year}{2025}).
\newblock


\bibitem[Liu et~al\mbox{.}(2025a)]%
        {liu2025spottrip}
\bibfield{author}{\bibinfo{person}{Yinghui Liu}, \bibinfo{person}{Hao Miao}, \bibinfo{person}{Guojiang Shen}, {et~al\mbox{.}}} \bibinfo{year}{2025}\natexlab{a}.
\newblock \showarticletitle{{SPOT-Trip: Dual-Preference Driven Out-of-Town Trip Recommendation}}. In \bibinfo{booktitle}{\emph{{NeurIPS}}}.
\newblock


\bibitem[Luo et~al\mbox{.}(2025)]%
        {luo2025fairgp}
\bibfield{author}{\bibinfo{person}{Renqiang Luo}, \bibinfo{person}{Huafei Huang}, \bibinfo{person}{Ivan Lee}, {et~al\mbox{.}}} \bibinfo{year}{2025}\natexlab{}.
\newblock \showarticletitle{Fair{GP}: A Scalable and Fair Graph Transformer Using Graph Partitioning}. In \bibinfo{booktitle}{\emph{Proceedings of the 39th Annual AAAI Conference on Artificial Intelligence}}.
\newblock


\bibitem[Luo et~al\mbox{.}(2024)]%
        {luo2024fairgt}
\bibfield{author}{\bibinfo{person}{Renqiang Luo}, \bibinfo{person}{Huafei Huang}, \bibinfo{person}{Shuo Yu}, {et~al\mbox{.}}} \bibinfo{year}{2024}\natexlab{}.
\newblock \showarticletitle{Fair{GT}: A Fairness-aware Graph Transformer}. In \bibinfo{booktitle}{\emph{Proceedings of the 32nd International Joint Conference on Artificial Intelligence}}.
\newblock


\bibitem[Ramp{\'a}{\v{s}}ek et~al\mbox{.}(2022)]%
        {rampavsek2022recipe}
\bibfield{author}{\bibinfo{person}{Ladislav Ramp{\'a}{\v{s}}ek}, \bibinfo{person}{Michael Galkin}, {et~al\mbox{.}}} \bibinfo{year}{2022}\natexlab{}.
\newblock \showarticletitle{Recipe for a general, powerful, scalable graph transformer}.
\newblock \bibinfo{journal}{\emph{Advances in Neural Information Processing Systems}} (\bibinfo{year}{2022}).
\newblock


\bibitem[Reimers and Gurevych(2019)]%
        {reimers2019sentence}
\bibfield{author}{\bibinfo{person}{Nils Reimers} {and} \bibinfo{person}{Iryna Gurevych}.} \bibinfo{year}{2019}\natexlab{}.
\newblock \showarticletitle{Sentence-BERT: Sentence Embeddings using Siamese BERT-Networks}. In \bibinfo{booktitle}{\emph{Proceedings of the 2019 Conference on Empirical Methods in Natural Language Processing and the 9th International Joint Conference on Natural Language Processing (EMNLP-IJCNLP)}}.
\newblock


\bibitem[Shi et~al\mbox{.}(2021)]%
        {shi2021masked}
\bibfield{author}{\bibinfo{person}{Yunsheng Shi}, \bibinfo{person}{Zhengjie Huang}, \bibinfo{person}{Shikun Feng}, {et~al\mbox{.}}} \bibinfo{year}{2021}\natexlab{}.
\newblock \showarticletitle{Masked Label Prediction: Unified Message Passing Model for Semi-Supervised Classification}. International Joint Conferences on Artificial Intelligence Organization.
\newblock


\bibitem[Tang et~al\mbox{.}(2024)]%
        {tang2024graphgpt}
\bibfield{author}{\bibinfo{person}{Jiabin Tang}, \bibinfo{person}{Yuhao Yang}, \bibinfo{person}{Wei Wei}, {et~al\mbox{.}}} \bibinfo{year}{2024}\natexlab{}.
\newblock \showarticletitle{Graphgpt: Graph instruction tuning for large language models}. In \bibinfo{booktitle}{\emph{Proceedings of the 47th International ACM SIGIR Conference on Research and Development in Information Retrieval}}.
\newblock


\bibitem[Touvron et~al\mbox{.}(2023)]%
        {touvron2023llama}
\bibfield{author}{\bibinfo{person}{Hugo Touvron}, \bibinfo{person}{Thibaut Lavril}, \bibinfo{person}{Gautier Izacard}, {et~al\mbox{.}}} \bibinfo{year}{2023}\natexlab{}.
\newblock \showarticletitle{Llama: Open and efficient foundation language models}.
\newblock \bibinfo{journal}{\emph{arXiv preprint arXiv:2302.13971}} (\bibinfo{year}{2023}).
\newblock


\bibitem[Vaswani et~al\mbox{.}(2017)]%
        {vaswani2017attention}
\bibfield{author}{\bibinfo{person}{Ashish Vaswani}, \bibinfo{person}{Noam Shazeer}, \bibinfo{person}{Niki Parmar}, {et~al\mbox{.}}} \bibinfo{year}{2017}\natexlab{}.
\newblock \showarticletitle{Attention is all you need}.
\newblock \bibinfo{journal}{\emph{Advances in neural information processing systems}} (\bibinfo{year}{2017}).
\newblock


\bibitem[Veli{\v{c}}kovi{\'c} et~al\mbox{.}(2018)]%
        {velivckovic2018graph}
\bibfield{author}{\bibinfo{person}{Petar Veli{\v{c}}kovi{\'c}}, \bibinfo{person}{Guillem Cucurull}, \bibinfo{person}{Arantxa Casanova}, {et~al\mbox{.}}} \bibinfo{year}{2018}\natexlab{}.
\newblock \showarticletitle{Graph Attention Networks}. In \bibinfo{booktitle}{\emph{International Conference on Learning Representations}}.
\newblock


\bibitem[Wang et~al\mbox{.}(2024b)]%
        {wang2024llms}
\bibfield{author}{\bibinfo{person}{Duo Wang}, \bibinfo{person}{Yuan Zuo}, \bibinfo{person}{Fengzhi Li}, {et~al\mbox{.}}} \bibinfo{year}{2024}\natexlab{b}.
\newblock \showarticletitle{Llms as zero-shot graph learners: Alignment of gnn representations with llm token embeddings}.
\newblock \bibinfo{journal}{\emph{Advances in Neural Information Processing Systems}}  \bibinfo{volume}{37} (\bibinfo{year}{2024}), \bibinfo{pages}{5950--5973}.
\newblock


\bibitem[Wang et~al\mbox{.}(2023)]%
        {wang2023can}
\bibfield{author}{\bibinfo{person}{Heng Wang}, \bibinfo{person}{Shangbin Feng}, \bibinfo{person}{Tianxing He}, {et~al\mbox{.}}} \bibinfo{year}{2023}\natexlab{}.
\newblock \showarticletitle{Can language models solve graph problems in natural language?}
\newblock \bibinfo{journal}{\emph{Advances in Neural Information Processing Systems}}  \bibinfo{volume}{36} (\bibinfo{year}{2023}), \bibinfo{pages}{30840--30861}.
\newblock


\bibitem[Wang et~al\mbox{.}(2025b)]%
        {wang2025ross3d}
\bibfield{author}{\bibinfo{person}{Haochen Wang}, \bibinfo{person}{Yucheng Zhao}, \bibinfo{person}{Tiancai Wang}, \bibinfo{person}{Haoqiang Fan}, \bibinfo{person}{Xiangyu Zhang}, {and} \bibinfo{person}{Zhaoxiang Zhang}.} \bibinfo{year}{2025}\natexlab{b}.
\newblock \showarticletitle{Ross3d: Reconstructive visual instruction tuning with 3d-awareness}. In \bibinfo{booktitle}{\emph{Proceedings of the IEEE/CVF International Conference on Computer Vision}}. \bibinfo{pages}{9275--9286}.
\newblock


\bibitem[Wang et~al\mbox{.}({[n.\,d.]})]%
        {wangreconstructive}
\bibfield{author}{\bibinfo{person}{Haochen Wang}, \bibinfo{person}{Anlin Zheng}, \bibinfo{person}{Yucheng Zhao}, \bibinfo{person}{Tiancai Wang}, \bibinfo{person}{Zheng Ge}, \bibinfo{person}{Xiangyu Zhang}, {and} \bibinfo{person}{Zhaoxiang Zhang}.} \bibinfo{year}{[n.\,d.]}\natexlab{}.
\newblock \showarticletitle{Reconstructive Visual Instruction Tuning}. In \bibinfo{booktitle}{\emph{The Thirteenth International Conference on Learning Representations}}.
\newblock


\bibitem[Wang et~al\mbox{.}(2025a)]%
        {wanggeneralization}
\bibfield{author}{\bibinfo{person}{Haoyu~Peter Wang}, \bibinfo{person}{Shikun Liu}, \bibinfo{person}{Rongzhe Wei}, {et~al\mbox{.}}} \bibinfo{year}{2025}\natexlab{a}.
\newblock \showarticletitle{Generalization Principles for Inference over Text-Attributed Graphs with Large Language Models}. In \bibinfo{booktitle}{\emph{Forty-second International Conference on Machine Learning}}.
\newblock


\bibitem[Wang et~al\mbox{.}(2024a)]%
        {wang2024instructgraph}
\bibfield{author}{\bibinfo{person}{Jianing Wang}, \bibinfo{person}{Junda Wu}, \bibinfo{person}{Yupeng Hou}, {et~al\mbox{.}}} \bibinfo{year}{2024}\natexlab{a}.
\newblock \showarticletitle{Instructgraph: Boosting large language models via graph-centric instruction tuning and preference alignment}.
\newblock \bibinfo{journal}{\emph{arXiv preprint arXiv:2402.08785}} (\bibinfo{year}{2024}).
\newblock


\bibitem[Wu et~al\mbox{.}(2022a)]%
        {wu2022semi}
\bibfield{author}{\bibinfo{person}{Fei Wu}, \bibinfo{person}{Xiao-Yuan Jing}, \bibinfo{person}{Pengfei Wei}, {et~al\mbox{.}}} \bibinfo{year}{2022}\natexlab{a}.
\newblock \showarticletitle{Semi-supervised multi-view graph convolutional networks with application to webpage classification}.
\newblock \bibinfo{journal}{\emph{Information Sciences}}  \bibinfo{volume}{591} (\bibinfo{year}{2022}), \bibinfo{pages}{142--154}.
\newblock


\bibitem[Wu et~al\mbox{.}(2023)]%
        {wu2023difformer}
\bibfield{author}{\bibinfo{person}{Qitian Wu}, \bibinfo{person}{Chenxiao Yang}, \bibinfo{person}{Wentao Zhao}, {et~al\mbox{.}}} \bibinfo{year}{2023}\natexlab{}.
\newblock \showarticletitle{Difformer: Scalable (graph) transformers induced by energy constrained diffusion}.
\newblock \bibinfo{journal}{\emph{arXiv preprint arXiv:2301.09474}} (\bibinfo{year}{2023}).
\newblock


\bibitem[Wu et~al\mbox{.}(2022b)]%
        {wu2022nodeformer}
\bibfield{author}{\bibinfo{person}{Qitian Wu}, \bibinfo{person}{Wentao Zhao}, \bibinfo{person}{Zenan Li}, {et~al\mbox{.}}} \bibinfo{year}{2022}\natexlab{b}.
\newblock \showarticletitle{Nodeformer: A scalable graph structure learning transformer for node classification}.
\newblock \bibinfo{journal}{\emph{Advances in Neural Information Processing Systems}} (\bibinfo{year}{2022}).
\newblock


\bibitem[Wu et~al\mbox{.}(2020)]%
        {wu2020comprehensive}
\bibfield{author}{\bibinfo{person}{Zonghan Wu}, \bibinfo{person}{Shirui Pan}, \bibinfo{person}{Fengwen Chen}, {et~al\mbox{.}}} \bibinfo{year}{2020}\natexlab{}.
\newblock \showarticletitle{A comprehensive survey on graph neural networks}.
\newblock \bibinfo{journal}{\emph{IEEE transactions on neural networks and learning systems}} (\bibinfo{year}{2020}).
\newblock


\bibitem[Yan et~al\mbox{.}(2024)]%
        {yan2024federated}
\bibfield{author}{\bibinfo{person}{Bo Yan}, \bibinfo{person}{Yang Cao}, \bibinfo{person}{Haoyu Wang}, {et~al\mbox{.}}} \bibinfo{year}{2024}\natexlab{}.
\newblock \showarticletitle{Federated heterogeneous graph neural network for privacy-preserving recommendation}. In \bibinfo{booktitle}{\emph{Proceedings of the ACM Web Conference 2024}}.
\newblock


\bibitem[Yan et~al\mbox{.}(2025a)]%
        {yan2025federated}
\bibfield{author}{\bibinfo{person}{Bo Yan}, \bibinfo{person}{Sihao He}, \bibinfo{person}{Cheng Yang}, {et~al\mbox{.}}} \bibinfo{year}{2025}\natexlab{a}.
\newblock \showarticletitle{Federated graph condensation with information bottleneck principles}. In \bibinfo{booktitle}{\emph{Proceedings of the AAAI Conference on Artificial Intelligence}}.
\newblock


\bibitem[Yan et~al\mbox{.}(2025b)]%
        {yan2025data}
\bibfield{author}{\bibinfo{person}{Bo Yan}, \bibinfo{person}{Zhongjian Zhang}, \bibinfo{person}{Huabin Sun}, \bibinfo{person}{Mengmei Zhang}, \bibinfo{person}{Yang Cao}, {and} \bibinfo{person}{Chuan Shi}.} \bibinfo{year}{2025}\natexlab{b}.
\newblock \showarticletitle{Data-centric federated graph learning with large language models}.
\newblock \bibinfo{journal}{\emph{arXiv preprint arXiv:2503.19455}} (\bibinfo{year}{2025}).
\newblock


\bibitem[Yang et~al\mbox{.}(2025)]%
        {yang2025harnessing}
\bibfield{author}{\bibinfo{person}{Jinyu Yang}, \bibinfo{person}{Ruijia Wang}, {et~al\mbox{.}}} \bibinfo{year}{2025}\natexlab{}.
\newblock \showarticletitle{Harnessing Language Model for Cross-Heterogeneity Graph Knowledge Transfer}. In \bibinfo{booktitle}{\emph{Proceedings of the AAAI Conference on Artificial Intelligence}}.
\newblock


\bibitem[Yang et~al\mbox{.}(2024)]%
        {yangvqgraph}
\bibfield{author}{\bibinfo{person}{Ling Yang}, \bibinfo{person}{Ye Tian}, \bibinfo{person}{Minkai Xu}, {et~al\mbox{.}}} \bibinfo{year}{2024}\natexlab{}.
\newblock \showarticletitle{VQGraph: Rethinking Graph Representation Space for Bridging GNNs and MLPs}. In \bibinfo{booktitle}{\emph{The Twelfth International Conference on Learning Representations}}.
\newblock


\bibitem[Yang et~al\mbox{.}(2016)]%
        {yang2016revisiting}
\bibfield{author}{\bibinfo{person}{Zhilin Yang}, \bibinfo{person}{William Cohen}, {and} \bibinfo{person}{Ruslan Salakhudinov}.} \bibinfo{year}{2016}\natexlab{}.
\newblock \showarticletitle{Revisiting semi-supervised learning with graph embeddings}. In \bibinfo{booktitle}{\emph{International conference on machine learning}}.
\newblock


\bibitem[Ye et~al\mbox{.}(2024)]%
        {ye2024language}
\bibfield{author}{\bibinfo{person}{Ruosong Ye}, \bibinfo{person}{Caiqi Zhang}, \bibinfo{person}{Runhui Wang}, {et~al\mbox{.}}} \bibinfo{year}{2024}\natexlab{}.
\newblock \showarticletitle{Language is All a Graph Needs}. In \bibinfo{booktitle}{\emph{EACL (Findings)}}.
\newblock


\bibitem[You et~al\mbox{.}(2020)]%
        {you2020graph}
\bibfield{author}{\bibinfo{person}{Yuning You}, \bibinfo{person}{Tianlong Chen}, \bibinfo{person}{Yongduo Sui}, {et~al\mbox{.}}} \bibinfo{year}{2020}\natexlab{}.
\newblock \showarticletitle{Graph contrastive learning with augmentations}.
\newblock \bibinfo{journal}{\emph{Advances in neural information processing systems}} (\bibinfo{year}{2020}).
\newblock


\bibitem[Yuan et~al\mbox{.}(2025)]%
        {yuan2025survey}
\bibfield{author}{\bibinfo{person}{Chaohao Yuan}, \bibinfo{person}{Kangfei Zhao}, \bibinfo{person}{Ercan~Engin Kuruoglu}, {et~al\mbox{.}}} \bibinfo{year}{2025}\natexlab{}.
\newblock \showarticletitle{A survey of graph transformers: Architectures, theories and applications}.
\newblock \bibinfo{journal}{\emph{arXiv preprint arXiv:2502.16533}} (\bibinfo{year}{2025}).
\newblock


\bibitem[Zhang et~al\mbox{.}(2024a)]%
        {zhang2024graphtranslator}
\bibfield{author}{\bibinfo{person}{Mengmei Zhang}, \bibinfo{person}{Mingwei Sun}, \bibinfo{person}{Peng Wang}, {et~al\mbox{.}}} \bibinfo{year}{2024}\natexlab{a}.
\newblock \showarticletitle{Graphtranslator: Aligning graph model to large language model for open-ended tasks}. In \bibinfo{booktitle}{\emph{Proceedings of the ACM Web Conference 2024}}. \bibinfo{pages}{1003--1014}.
\newblock


\bibitem[Zhang et~al\mbox{.}(2025a)]%
        {zhang2025can}
\bibfield{author}{\bibinfo{person}{Zhongjian Zhang}, \bibinfo{person}{Xiao Wang}, {et~al\mbox{.}}} \bibinfo{year}{2025}\natexlab{a}.
\newblock \showarticletitle{Can large language models improve the adversarial robustness of graph neural networks?}. In \bibinfo{booktitle}{\emph{Proceedings of the 31st ACM SIGKDD Conference on Knowledge Discovery and Data Mining V. 1}}.
\newblock


\bibitem[Zhang et~al\mbox{.}(2024b)]%
        {zhang2024endowing}
\bibfield{author}{\bibinfo{person}{Zhongjian Zhang}, \bibinfo{person}{Mengmei Zhang}, {et~al\mbox{.}}} \bibinfo{year}{2024}\natexlab{b}.
\newblock \showarticletitle{Endowing pre-trained graph models with provable fairness}. In \bibinfo{booktitle}{\emph{Proceedings of the ACM Web Conference 2024}}.
\newblock


\bibitem[Zhang et~al\mbox{.}(2025b)]%
        {zhang2025rethinking}
\bibfield{author}{\bibinfo{person}{Zhongjian Zhang}, \bibinfo{person}{Mengmei Zhang}, \bibinfo{person}{Xiao Wang}, {et~al\mbox{.}}} \bibinfo{year}{2025}\natexlab{b}.
\newblock \showarticletitle{Rethinking Byzantine Robustness in Federated Recommendation from Sparse Aggregation Perspective}. In \bibinfo{booktitle}{\emph{Proceedings of the AAAI Conference on Artificial Intelligence}}.
\newblock


\bibitem[Zhao et~al\mbox{.}(2024)]%
        {zhaographtext}
\bibfield{author}{\bibinfo{person}{Jianan Zhao}, \bibinfo{person}{Le Zhuo}, \bibinfo{person}{Yikang Shen}, {et~al\mbox{.}}} \bibinfo{year}{2024}\natexlab{}.
\newblock \showarticletitle{GraphText: Graph Reasoning in Text Space}. In \bibinfo{booktitle}{\emph{Adaptive Foundation Models: Evolving AI for Personalized and Efficient Learning}}.
\newblock


\bibitem[Zhou et~al\mbox{.}(2024)]%
        {zhou2024unifying}
\bibfield{author}{\bibinfo{person}{Cai Zhou}, \bibinfo{person}{Xiyuan Wang}, {and} \bibinfo{person}{Muhan Zhang}.} \bibinfo{year}{2024}\natexlab{}.
\newblock \showarticletitle{Unifying generation and prediction on graphs with latent graph diffusion}.
\newblock \bibinfo{journal}{\emph{Advances in Neural Information Processing Systems}} (\bibinfo{year}{2024}).
\newblock


\bibitem[Zhou et~al\mbox{.}(2025)]%
        {zhou-etal-2025-graph}
\bibfield{author}{\bibinfo{person}{Huachi Zhou}, \bibinfo{person}{Jiahe Du}, {et~al\mbox{.}}} \bibinfo{year}{2025}\natexlab{}.
\newblock \showarticletitle{Each graph is a new language: Graph Learning with {LLM}s}. In \bibinfo{booktitle}{\emph{Findings of the Association for Computational Linguistics: ACL 2025}}. \bibinfo{publisher}{Association for Computational Linguistics}.
\newblock


\bibitem[Zhu et~al\mbox{.}(2020)]%
        {zhu2020deep}
\bibfield{author}{\bibinfo{person}{Yanqiao Zhu}, \bibinfo{person}{Yichen Xu}, {et~al\mbox{.}}} \bibinfo{year}{2020}\natexlab{}.
\newblock \showarticletitle{{Deep Graph Contrastive Representation Learning}}. In \bibinfo{booktitle}{\emph{ICML Workshop on Graph Representation Learning and Beyond}}.
\newblock


\end{thebibliography}
\clearpage
\appendix
\section{Related Work}
\subsection{Learning on Graphs}
Graph data are ubiquitous in the real world, including social networks~\cite{huang2024can, zhang2024endowing}, webpage networks~\cite{wu2022semi}, citation networks~\cite{yang2016revisiting, yan2025federated}, and e-commerce networks~\cite{yan2024federated, zhang2025rethinking, liu2025spottrip, chen2025sustainability}. However, due to their inherent irregularity and complex connectivity patterns, effectively modeling graph structure is a fundamental challenge.
Graph Neural Networks (GNNs) have emerged as powerful tools to address this issue by leveraging the message-passing mechanism to aggregate information from local neighborhoods~\cite{wu2020comprehensive,guo2024data}. Representative models such as GCN~\cite{DBLP:conf/iclr/KipfW17}, GAT~\cite{velivckovic2018graph} and GraphSAGE~\cite{hamilton2017inductive} have shown notable success in node classification and link prediction tasks. 
Concurrently, there is a growing interest in adapting transformer-based architectures~\cite{vaswani2017attention} to graph domains~\cite{yuan2025survey,wu2023difformer,wu2022nodeformer, luo2024fairgt, luo2025fairgp}. These models aim to capture global dependencies and richer structural semantics, providing a complementary perspective to traditional local aggregation schemes. 
Despite these successes, a critical limitation of existing GNNs lies in their poor task generalization. Most existing models are trained in a task-specific and dataset-dependent manner, which hampers their performance when transferred to different distributions or tasks. 

\subsection{Large Language Models for Graphs}
Recently, the remarkable success of large language models (LLMs)~\cite{touvron2023llama,achiam2023gpt,chiang2023vicuna} has motivated researchers to adapt them to the graph learning domain, with the ultimate goal of developing a graph foundation model that generalizes across diverse scenarios~\cite{liu2025graph, yang2025harnessing,guograph,yan2025data}. Typically, LLMs are employed as predictors to accomplish various graph tasks, attracting considerable attention~\cite{jin2024large}. A key challenge lies in aligning graph data with the natural language embedding space, so that LLMs can better comprehend graphs. To address this challenge, existing efforts mainly fall into Graph-Textualizing LLMs (GTextLLMs) and Graph-Tokenizing LLMs (GTokenLLMs).

GTextLLMs~\cite{chen2024exploring,fatemitalk,guo2023gpt4graph,wang2023can,zhaographtext,wang2024instructgraph} utilize natural language to describe graphs via manually crafted prompt templates, directly aligning them with text. For instance, NLGraph~\cite{wang2023can} represents graphs as node and edge lists embedded in textual templates. GraphText~\cite{zhaographtext} converts graphs into tree-based textual formats using syntactic graph parsing. 
However, such strategies often suffer from redundancy and imprecise characterization~\cite{chen2024exploring,chen2024llaga,tang2024graphgpt}.
In contrast, GTokenLLMs~\cite{chen2024llaga,tang2024graphgpt,zhang2024graphtranslator,wang2024llms,konggofa} encode complex structures and lengthy texts into a graph token sequence, and treat them as prefix tokens for querying LLMs. To align graph tokens with the text embedding space, they typically construct various language instructions related to graph tasks, and learn a projector to facilitate modalities alignment.
For example, 
GraphGPT~\cite{tang2024graphgpt} builds language instructions around the graph matching task, while LLaGA~\cite{chen2024llaga} builds them by node classification, link prediction, and node description tasks.
Despite these successes, their alignment process follows an LLM-centric design that relies solely on text supervision, resulting in only implicit graph-text alignment and thus limiting performance.

\section{Proofs}
\subsection{Proof of Proposition~\ref{prop:raw_graph_recon}}\label{proof:raw_graph_recon}
\begin{proof}
Since $\mathbf{H}$ is a deterministic function of $\boldsymbol{s}^{G}$ (i.e., mean pooling), the data processing inequality ensures that $\mathcal{I}(\mathcal{G};\boldsymbol{s}^{G}) \ge \mathcal{I}(\mathcal{G};\mathbf{H})$.
By decoding $\mathbf{H}$ back into node features $\mathbf{Z}$ and adjacency matrix $\mathbf{A}$, where $\mathbf{A}_{ij}=1$ if nodes $v_i$ and $v_j$ are connected, otherwise $\mathbf{A}_{ij}=0$, the reconstruction loss can be expressed as:
\begin{align}
\mathcal{L}_{\text{graph}}
=\lambda_f\,\mathbb{E}\big[-\log q_{f}(\mathbf{Z}\!\mid\! \mathbf{H})\big]
+\lambda_s\,\mathbb{E}\big[-\log q_s(\mathbf{A}\!\mid\! \mathbf{H})\big].
\end{align}
Assume $q_f(\mathbf{Z}\!\mid\! \mathbf{H})=\mathcal{N}(\mu_f(\mathbf{H}),\sigma^2\mathbf{I})$, which corresponds to MSE loss over features; and $q_s(\mathbf{A}\!\mid\! \mathbf{H})=\prod_{i<j}\mathrm{Bernoulli}(\pi_s(\mathbf{h}_i,\mathbf{h}_j))$, which corresponds to BCE loss over edges. By the Barber–Agakov identity, for any conditional model $q$ we have:
\begin{align}
\mathbb{E}[-\log q_f(\mathbf{Z}\!\mid\! \mathbf{H})]
&= \mathcal{H}(\mathbf{Z}\!\mid\! \mathbf{H})
+\mathbb{E}[\mathrm{KL}(p(\mathbf{Z}\!\mid\! \cdot)\,\|\,q_f(\mathbf{Z}\!\mid\! \cdot))]
\nonumber \\
&\ge \mathcal{H}(\mathbf{Z}\!\mid\! \mathbf{H}), \\
\mathbb{E}[-\log q_s(\mathbf{A}\!\mid\! \mathbf{H})]
&= \mathcal{H}(\mathbf{A}\!\mid\! \mathbf{H})
+\mathbb{E}[\mathrm{KL}(p(\mathbf{A}\!\mid\! \cdot)\,\|\,q_s(\mathbf{A}\!\mid\! \cdot))]
\nonumber \\
&\ge \mathcal{H}(\mathbf{A}\!\mid\! \mathbf{H}).
\end{align}
Thus,
\begin{align}
\mathcal{L}_{\text{graph}}
\ \ge\ \lambda_f \mathcal{H}(\mathbf{Z}\!\mid\! \mathbf{H})
+ \lambda_s \mathcal{H}(\mathbf{A}\!\mid\! \mathbf{H}).
\end{align}
Here, we assume the graph as a pair of independent random variables $(\mathbf{Z}, \mathbf{A})$. The mutual information
can be decomposed as:
\begin{align}
\mathcal{I}(\mathcal{G};\mathbf{H})
&=\mathcal{I}(\mathbf{Z};\mathbf{H})
+\mathcal{I}(\mathbf{A};\mathbf{H}) \nonumber \\
&= \mathcal{H}(\mathbf{Z})-\mathcal{H}(\mathbf{Z}\!\mid\! \mathbf{H})
+ \mathcal{H}(\mathbf{A})-\mathcal{H}(\mathbf{A}\!\mid\! \mathbf{H}).
\end{align}
Since $\mathcal{H}(\mathbf{Z})$ and $\mathcal{H}(\mathbf{A})$ are parameter-independent, combining the above inequality and dropping non-negative KL terms yields:
\begin{align}
\mathcal{I}(\mathcal{G};\boldsymbol{s}^{G})
\ \ge\ \mathcal{H}(\mathbf{Z})+\mathcal{H}(\mathbf{A})
\ -\ \tfrac{1}{\lambda_f}\mathcal{L}_{\text{feat}}
\ -\ \tfrac{1}{\lambda_s}\mathcal{L}_{\text{topo}}.
\end{align}
\end{proof}
The proof indicates that minimizing the raw graph reconstruction loss maximizes a computable lower bound of $\mathcal{I}(\mathcal{G};\boldsymbol{s}^{G})$, which directly improves the upper bound of optimization objective C-GTMI $\mathcal{I}(\boldsymbol{x}_t;\boldsymbol{s}^{G}\!\mid\! \boldsymbol{s}^{T}_{<t})$.

\subsection{Proof of Proposition~\ref{prop:latent_sim}}\label{proof:latent_sim}

\begin{proof}
Given a subgraph $\mathcal{G}$, let $\mathbf{E}=g(\mathcal{G})$ denote the latent node representation matrix extracted by the pre-trained GNN $g$. The reconstruction tokens $\mathbf{H}$ are mapped to $\mathbf{E}$ by a lightweight MLP $s(\cdot)$, with normalization ensuring $\|\mathbf{E}\|=\|s(\mathbf{H})\|=1$. The reconstruction loss is defined as the cosine similarity:
\begin{equation}
\mathcal L_{\text{sim}}=\mathbb E[1-\cos \left(s(\mathbf{H}),\mathbf{E}\right)].
\end{equation}
On the unit sphere, we assume that $\mathbf{E}$ conditioned on $\mathbf{H}$ follows a von Mises–Fisher (vMF) distribution:
\begin{equation}
q_s(\mathbf{E}\!\mid\! \mathbf{H})=\text{vMF}(\mathbf{E};\ \mu(\mathbf{H}),\kappa)
\propto \exp\big(\kappa\cdot\lambda_{l}\cdot \cos(\mu(\mathbf{H}),\mathbf{E})\big),
\end{equation}
where $\mu(\mathbf{H}):=\frac{s(\mathbf{H})}{\|s(\mathbf{H})\|}$, $\kappa>0$, $\lambda_{l}$ is a scaling factor. It follows that:
\begin{equation}
-\log q_s(\mathbf{E}\!\mid\! \mathbf{H})= -\kappa\cdot\lambda_{l}\cdot\cos \left(s(\mathbf{H}),\mathbf{E}\right) + \log C_d(\kappa),
\end{equation}
where $\log C_d(\kappa)$ is a constant independent of $\mathbf{H}$. Therefore,
\begin{equation}
\begin{aligned}
\mathbb E[-\log q_s(\mathbf{E}\!\mid\! \mathbf{H})]&=\text{const}+\kappa\cdot\lambda_{l}\cdot\mathbb E[1-\cos\left(s(\mathbf{H}),\mathbf{E}\right)]\\
&=\text{const}+\kappa\cdot\lambda_{l}\cdot\mathcal{L}_{\text{sim}}.
\end{aligned}
\end{equation}
For any conditional model $q_s(\mathbf{E}\!\mid\! \mathbf{H})$, we have:
\begin{equation}
\begin{aligned}
\mathbb E[-\log q_s(\mathbf{E}\!\mid\! \mathbf{H})]
&=\mathcal{H}(\mathbf{E}\!\mid\! \mathbf{H})+\mathbb E[\mathrm{KL}(p(\mathbf{E}\!\mid\! \cdot)\,\|\,q_s(\mathbf{E}\!\mid\! \cdot))] \\
&\ge \mathcal{H}(\mathbf{E}\!\mid\! \mathbf{H}).
\end{aligned}
\end{equation}
Thus, we obtain:
\begin{equation}
\mathcal L_{\text{sim}}\ge \frac{1}{\kappa\cdot\lambda_{l}}\mathcal{H}(\mathbf{E}\!\mid\! \mathbf{H})- \text{const}\Rightarrow 
\mathcal{H}(\mathbf{E}\!\mid\! \mathbf{H})\le \kappa\cdot\lambda_{l}\cdot\mathcal L_{\text{sim}}+\text{const}.
\end{equation}
Since $\mathcal{I}(\mathbf{E};\mathbf{H})=\mathcal{H}(\mathbf{E})-\mathcal{H}(\mathbf{E}\!\mid\! \mathbf{H})$, we obtain:
\begin{equation}
\mathcal{I}(\mathbf{E};\mathbf{H})\ \ge\ \mathcal{H}(\mathbf{E})-(\kappa\cdot\lambda_{l}\cdot\mathcal L_{\text{sim}}+\text{const}).
\end{equation}
Moreover, $g$ is frozen so that $\mathbf{E}=g(\mathcal{G})$ is a deterministic function of $\mathcal{G}$, and $\mathbf{H}$ is a deterministic function of $\boldsymbol{s}^{G}$, thus we have:
\begin{equation}
\mathcal{I}(\mathcal{G};\boldsymbol{s}^{G})
\!\ge\! \mathcal{I}(\mathcal{G};\mathbf{H})
\!\ge\! \mathcal{I}(\mathbf{E};\mathbf{H})
\!\ge\! \mathcal{H}(\mathbf{E})-(\kappa\!\cdot\!\lambda_{l}\!\cdot\!\mathcal L_{\text{sim}}+\text{const}).
\end{equation}
\end{proof}
The proof indicates that minimizing the similarity loss between $\mathbf{H}$ and $\mathbf{E}$ maximizes a computable lower bound of $\mathcal{I}(\mathcal{G};\boldsymbol{s}^{G})$, which in turn improves the attainable upper bound of C-GTMI $\mathcal{I}(\boldsymbol{x}_t;\boldsymbol{s}^{G}\!\mid\! \boldsymbol{s}^{T}_{<t})$.

\subsection{Proof of Proposition~\ref{prop:latent_denoise}}\label{proof:latent_denoise}

\begin{proof}
Given a subgraph $\mathcal{G}$, let $\mathbf{E}=g(\mathcal{G})$ denote the latent node representation matrix extracted by the pre-trained GNN $g$. The reconstruction tokens $\mathbf{H}$ are projected into a condition matrix $\mathbf{C}=\Psi(\mathbf{H})$. 
The objective is to model a conditional distribution $q_f(\mathbf{E}_{t-1}\!\mid\! \mathbf{E}_{t}, \mathbf{C})$, such that the clean embeddings $\mathbf{E}$ can be recovered from their noisy counterparts $\mathbf{E}_{t}$ under the condition $\mathbf{C}$.
The forward diffusion process is formulated as a Markov chain that gradually perturbs the clean embeddings $\mathbf{E}$ by adding Gaussian noise, thereby constructing noisy embeddings $\mathbf{E}_{t}$:
\begin{equation}
\begin{aligned}
&p(\mathbf{E}_t\!\mid\! \mathbf{E})=\mathcal N\!\big(\sqrt{\bar\alpha_t}\,\mathbf{E},\ (1-\bar\alpha_t)\,\mathbf{I}\big),\\
&\mathbf{E}_t=\sqrt{\bar\alpha_t}\mathbf{E}+\sqrt{1-\bar\alpha_t}\,\boldsymbol{\epsilon},\ \quad \boldsymbol{\epsilon}\sim\mathcal N(\mathbf{0},\mathbf{I}),
\end{aligned}
\end{equation}
where $\mathcal{N}(\boldsymbol{\mu},\boldsymbol{\Sigma})$ denotes a Gaussian distribution, $t$ indicates discrete timesteps, $\alpha_t=1-\beta_t$, $\bar{\alpha}_t=\prod_{i=1}^t\alpha_t$, and $\beta_t \in (0,1)$ is a pre-defined variance schedule. A latent diffusion model learns to reverse this progressive noise addition process. Specifically, to iteratively generate clean embeddings $\mathbf{E}$ from pure noise $\mathbf{E}_t$ conditioned on $\mathbf{C}$, the reverse process is parameterized as
\begin{equation}
\begin{aligned}
&q_f(\mathbf{E}_{t-1}\!\mid\! \mathbf{E}_t, \mathbf{C})=\mathcal N\!\left(\boldsymbol{\mu}_f(\mathbf{E}_t,\mathbf{C},t),\ \sigma_t^2 \mathbf{I}\right),\\
&\boldsymbol{\mu}_f=\frac{1}{\sqrt{\alpha_t}}\left(\mathbf{E}_t-\frac{1-\alpha_t}{\sqrt{1-\bar\alpha_t}}\,\boldsymbol{\epsilon}_f(\mathbf{E}_t;\mathbf{C},t)\right).
\end{aligned}
\end{equation}
The marginal conditional likelihood is then given by
\begin{equation}
\begin{aligned}
\log q_f(\mathbf{E}\!\mid\! \mathbf{H})
&= \log \int q_f(\mathbf{E}_T)\,
   \prod_{t=1}^{T} q_f(\mathbf{E}_{t-1}\!\mid\! \mathbf{E}_t,\mathbf{C}) \\
&\quad \times q_f(\mathbf{E}\!\mid\! \mathbf{E}_1,\mathbf{C})\, d \mathbf{E}_{1:T}.
\end{aligned}
\end{equation}
By introducing the forward distribution $p(\mathbf{E}_{1:T}\!\mid\! \mathbf{E})$ as a variational approximation, the standard variational lower bound (ELBO) can be derived:
\begin{equation}
\begin{aligned}
\log q_f(\mathbf{E}\!\mid\! \mathbf{H}) 
&\ge -\mathrm{KL}\left(p(\mathbf{E}_T\!\mid\! \mathbf{E})\,\|\,q_f(\mathbf{E}_T)\right) \\
&-\sum_{t=2}^{T}\mathbb E\left[\mathrm{KL}\big(p(\mathbf{E}_{t-1}\!\mid\! \mathbf{E}_t,\mathbf{E})\,\|\,q_f(\mathbf{E}_{t-1}\!\mid\! \mathbf{E}_t,\mathbf{C})\big)\right] \\
&+\mathbb E\left[\log q_f(\mathbf{E}\!\mid\! \mathbf{E}_1,\mathbf{C})\right].
\end{aligned}
\end{equation}
Under a variance-preserving schedule and Gaussian parameterization, each KL term can be rewritten as a mean-squared error (MSE) on the noise prediction, up to a constant independent of the parameters:
\begin{equation}
\begin{aligned}
\mathrm{KL}\!\big(p(\mathbf{E}_{t-1}\!\!\mid\!\! \mathbf{E}_t,\mathbf{E})\|q_f(\mathbf{E}_{t-1}\!\!\mid\!\!\mathbf{E}_t,\mathbf{C})\big) 
&= \tilde{\lambda}_t\mathbb{E}_{\boldsymbol{\epsilon}\sim\mathcal N(\mathbf 0,\mathbf I)} 
\| \boldsymbol{\epsilon}_f(\mathbf{E}_t;\mathbf{C},t)\!-\!\boldsymbol{\epsilon}\|_2^2 \\
&+ \text{const}_t,
\end{aligned}
\end{equation}
where $\tilde{\lambda}_t>0$ depends on the noise schedule. Substituting back into the ELBO and aggregating constants yields the diffusion loss, which is equivalent to the negative ELBO up to a constant:
\begin{equation}
\mathcal{L}_{\text{diff}}
=\sum_{t=1}^T \lambda_t\;
\mathbb{E}_{(\mathbf{E},\mathbf{H}),\,\boldsymbol{\epsilon}}
\Big\| \boldsymbol{\epsilon}_f\!\big(\sqrt{\bar\alpha_t}\mathbf{E}+\sqrt{1-\bar\alpha_t}\boldsymbol{\epsilon};\ \mathbf{C},t\big) -\boldsymbol{\epsilon} \Big\|_2^2 ,
\end{equation}
where $\lambda_t>0$ absorbs the constants into the weights. This equivalence leads to the upper bound on the negative log-likelihood:
\begin{equation}
\mathbb{E}[-\log q_f(\mathbf{E}\!\mid\! \mathbf{H})] \;\le\; \mathrm{const}+\lambda_{l}\cdot\mathcal{L}_{\text{diff}},
\end{equation}
where $\lambda_{l}$ is a scaling factor. Moreover, for any conditional model $q_f(\mathbf{E}\!\mid\! \mathbf{H})$ we have:
\begin{equation}
\begin{aligned}
\mathbb{E}[-\log q_f(\mathbf{E}\!\mid\! \mathbf{H})]
&= \mathcal{H}(\mathbf{E}\!\mid\! \mathbf{H})
+ \mathbb{E}\!\left[ \mathrm{KL}\!\big(p(\mathbf{E}\!\mid\! \mathbf{H})\ \|\ q_f(\mathbf{E}\!\mid\! \mathbf{H})\big) \right] \\
&\ge \mathcal{H}(\mathbf{E}\!\mid\! \mathbf{H}).
\end{aligned}
\end{equation}
Discarding the non-negative KL term yields the following upper bound on the conditional entropy:
\begin{equation}
\mathcal{H}(\mathbf{E}\!\mid\! \mathbf{H}) \ \le\ \mathrm{const}+\lambda_{l}\cdot\mathcal{L}_{\text{diff}}.
\end{equation}
By the definition $\mathcal{I}(\mathbf{E};\mathbf{H})=\mathcal{H}(\mathbf{E})-\mathcal{H}(\mathbf{E}\!\mid\! \mathbf{H})$, we obtain:
\begin{equation}
\mathcal{I}(\mathbf{E};\mathbf{H}) \ \ge\ \mathcal{H}(\mathbf{E})-(\mathrm{const}+\lambda_{l}\cdot\mathcal{L}_{\text{diff}}).
\end{equation}
Finally, since $\mathbf{E}=g(\mathcal{G})$ is a deterministic function of $\mathcal{G}$ and $\mathbf{H}$ is deterministically obtained from $\boldsymbol{s}^{G}$ by mean pooling, we have:
\begin{equation}
\mathcal{I}(\mathcal{G};\boldsymbol{s}^{G})\ge \mathcal{I}(\mathbf{E};\boldsymbol{s}^{G}) \ge \mathcal{I}(\mathbf{E};\mathbf{H}) \ge \mathcal{H}(\mathbf{E})-(\mathrm{const}+\lambda_{l}\cdot\mathcal{L}_{\text{diff}}).
\end{equation}
\end{proof}
The proof indicates that minimizing the diffusion noise prediction loss $\mathcal{L}_{\text{diff}}$ maximizes a computable lower bound of $\mathcal{I}(\mathcal{G};\boldsymbol{s}^{G})$, which in turn improves the attainable upper bound of C-GTMI $\mathcal{I}(\boldsymbol{x}_t;\boldsymbol{s}^{G}\!\mid\! \boldsymbol{s}^{T}_{<t})$.

\section{Experiment Details}
\subsection{Dataset Descriptions}\label{app: exp-dataset}
We conduct experiments on four widely-used datasets: Cora~\cite{yang2016revisiting}, Pubmed~\cite{yang2016revisiting} and OGBN-Arxiv~\cite{hu2020open} from citation networks, as well as Reddit~\cite{huang2024can} from a social network. The dataset statistics are in Table~\ref{tab:dataset_stats}, and the detailed descriptions are listed as follows:\\
$\bullet$ \textbf{Cora} is a citation network in which each node denotes a research paper in the field of computer science, and the associated raw text features include the title and abstract of the paper. Edges represent citation relationships between papers. Each node is assigned a label indicating the paper's category.\\
$\bullet$ \textbf{Pubmed} is a citation network consisting of research papers in the biomedical domain, where each node corresponds to a paper and each edge denotes a citation relationship between two papers.\\
$\bullet$ \textbf{OGBN-Arxiv} is a citation network comprising research papers and their citation relationships, collected from the arXiv repository. Each node represents a paper, and each edge indicates a citation link between two papers.\\
$\bullet$ \textbf{Reddit} is a social network in which each node represents a user, and edges indicate whether two users have interacted by replying to each other. The raw text associated with each node comprises the content from the user’s previously published subreddit posts, limited to the most recent three entries. Each node is labeled to indicate whether the user is categorized as popular or normal.

\begin{table}[t]
  \centering
  \caption{Statistics of experiment datasets.}
  \vskip -0.15in
  \label{tab:dataset_stats}
  \resizebox{\linewidth}{!}{
  \begin{tabular}{cccccc}
    \toprule
    Dataset & \# Nodes & \# Edges & \# Class & Splitting & Domain \\
    \midrule
    Cora       & 2,708   & 5,429    & 7  & 6:2:2 & citation \\
    Pubmed     & 19,717  & 44,338   & 3  & 6:2:2 & citation \\
    OGBN-Arxiv      & 169,343 & 1,166,243 & 40 & 6:2:3 & citation \\
    Reddit  & 33,434  & 198,448  & 2  & 1:1:8 & social network \\
    \bottomrule
  \end{tabular}
  }
\vskip -0.15in
\end{table}

\subsection{Baseline Descriptions}\label{sec:exp-baseline}
The traditional GNN models include:
\begin{itemize}[leftmargin=*]
    \item \textbf{GCN}~\cite{DBLP:conf/iclr/KipfW17} is a popular graph convolutional network based on spectral theory.
    \item \textbf{GAT}~\cite{velivckovic2018graph} is composed of multiple attention layers, which can learn different weights for different neighborhood nodes.
    \item \textbf{GraphSAGE}~\cite{hamilton2017inductive} inductively learns node embeddings by utilizing ego-node features along with a fixed number of neighboring node features.
\end{itemize}
The transformer-based graph models include:
\begin{itemize}[leftmargin=*]
    \item \textbf{UniMP}~\cite{shi2021masked} employs a Graph Transformer network to integrate both feature and label propagation during training and inference.
    \item \textbf{NodeFormer}~\cite{wu2022nodeformer} filters out dissimilar neighbors using a fixed threshold and aggregates multi-view neighborhood information to address graph inconsistency.
    \item \textbf{DIFFormer}~\cite{wu2023difformer} introduces an energy-constrained, diffusion-based graph transformer designed to encode a batch of instances into evolutionary states by incorporating contextual information from other instances.
\end{itemize}
The pretrain-finetune GNNs include:
\begin{itemize}[leftmargin=*]
    \item \textbf{GraphCL}~\cite{you2020graph} employs diverse graph augmentation techniques to generate multiple views of a graph and learns node representations by maximizing feature consistency across these views.
    \item \textbf{GRACE}~\cite{zhu2020deep} proposes to augment the graph structure by both random edge dropout and random node feature dropping, and uses the corrupted graphs as the contrastive views.
    \item \textbf{LGD}~\cite{zhou2024unifying} LGD pretrains an autoencoder to embed graphs into a powerful latent space and then trains a diffusion model in this space to unify graph generation and prediction tasks.
\end{itemize}
The GTextLLMs include:
\begin{itemize}[leftmargin=*]
    \item \textbf{GraphText-ICL}~\cite{zhaographtext} is the in-context learning variant of GraphText, which eliminates the tuning process. GraphText decouples depth and scope by encapsulating node attributes and relationships in the graph syntax tree and processing it with an LLM. 
\end{itemize}
The GTokenLLMs include:
\begin{itemize}[leftmargin=*]
    \item \textbf{GraphText-SFT}~\cite{zhaographtext} is a supervised fine-tuning variant of GraphText, which employs an MLP projector to map node continuous features into the textual input space for querying LLMs. 
    \item \textbf{GraphGPT}~\cite{tang2024graphgpt} encodes graph data using a pre-trained GNN and introduces a dual-stage instruction tuning to align graph representations with natural language. 
    \item \textbf{LLaGA}~\cite{chen2024llaga} designs two graph description templates to serialize graph data and instruction-tuning LLMs with three different tasks, including node description, node classification, and link prediction tasks.
\end{itemize}

\subsection{Implementation Details}\label{sec:exp-implementation}
First, RGLM aims to explicitly incorporate graph supervision into the tuning process, rather than designing a new instruction dataset. Therefore, we utilize the instruction datasets from LLaGA~\cite{chen2024llaga} to train the model.
In the tuning stage, we train our model for one epoch with the AdamW optimizer, set the batch size to 4 per device, and set the projector learning rate to 2e-3. Following LLaGA~\cite{chen2024llaga}, to compensate for the limited data size, we replicate the training samples from the smallest dataset, Cora and Reddit, three times.
The warmup ratio is set to 3e-2, and the maximum input length of LLMs is 4096. For the Neighbor Detail Template, we sample two-hop neighbors around each node, setting the sample size to 10 for each hop. We adopt LoRA~\cite{hu2022lora} for efficient fine-tuning, with the learning rate, dropout rate, rank, and $\alpha$ set to 5e-4, 0.1, 8, and 32, respectively. 
In the inference stage, we set the temperature of the LLM to 0.001 to ensure reproducibility of the results.
We uniformly set the learning rate to 5e-4 for all three RGLM variants and configure the hyperparameters as follows: For RGLM-Decoder, the hyperparameter $\lambda_f$ is selected from $\{0.1, 0.2, 0.4, 0.6, 0.8, 1.0\}$, and $\lambda_s$ is selected from $\{1, 2, 4, 6, 8, 10\}$. For RGLM-Similarizer and RGLM-Denoiser, $\lambda_l$ is chosen from 0.2 to 2.0 with a step size of 0.2. 

For the LM-based pre-trained GNN encoder, following~\cite{zhou2024unifying,rampavsek2022recipe}, we use the combination of DeeperGCN~\cite{li2020deepergcn} and GAT~\cite{velivckovic2018graph} as the backbone, with 7 layers and 96 hidden dimensions. We use LapPE and RWSE. We train the model for 500 epochs with a masked label ratio of 0.8, and a cosine learning rate with 50 warm-up epochs and an initial value of 1e-3.

To facilitate fair comparisons, we adopt a unified text encoder, Sentence-BERT~\cite{reimers2019sentence}, to extract node features from the raw textual attributes. 
All GTokenLLMs use Vicuna-7B-v1.5-16K~\cite{chiang2023vicuna} as the base LLM. 
For these baseline models, we implement them with the source code provided by the authors.
Unless otherwise specified, we adopt the default parameter setting in the author’s implementation.
For datasets not included in the released baseline code, we tune all baselines using a grid search strategy.
For all experiments, we select the optimal hyperparameters on the validation set and apply them to the test set. The implementation of all models uses PyG~\cite{fey2019fast}.
\subsection{Computing Environment and Resources}\label{exp: resource_cost}
The experiments are conducted in a computing environment with the following specifications: 
\begin{itemize}[leftmargin=*]
    \item \textbf{OS}: Linux ubuntu 5.15.0-102-generic. 
    \item \textbf{CPU}: Intel(R) Xeon(R) Platinum 8358 CPU @ 2.60GHz. 
    \item \textbf{GPU}: NVIDIA A800 80GB.
\end{itemize}



\end{document}